\documentclass{article} % For LaTeX2e

\usepackage{etoolbox}
\newtoggle{withappendix}
\newtoggle{arxiv}
\toggletrue{withappendix}
\togglefalse{arxiv}
\toggletrue{withappendix}\toggletrue{arxiv}

\iftoggle{arxiv}{
  \usepackage{times}
  \setlength{\textwidth}{6.5in}
  \setlength{\textheight}{9in}
  \setlength{\oddsidemargin}{0in}
  \setlength{\evensidemargin}{0in}
  \setlength{\topmargin}{-0.5in}
  \newlength{\defbaselineskip}
  \setlength{\defbaselineskip}{\baselineskip}
  \setlength{\marginparwidth}{0.8in}
}{
  \usepackage{nips15submit_e,times}
  \nipsfinalcopy
}

\usepackage{hyperref}
\usepackage{url}

\usepackage[numbers,sort]{natbib}
\usepackage{multibib}
\newcites{sec}{Secondary Literature}

\patchcmd{\thebibliography}{\section*{\refname}}{}{}{}

\usepackage{amsopn,amsmath,amsthm,amssymb}

\usepackage{graphicx}

\usepackage{algorithm}
\usepackage{algorithmic}

\usepackage{bbold}

\usepackage{subfigure}

\usepackage{breqn}

\usepackage{tabu}

\usepackage{tikz}
\usetikzlibrary{shapes,backgrounds}

\usepackage{url}

\usepackage{thmtools, thm-restate}

\definecolor{blue-violet}{rgb}{0.54, 0.17, 0.89}

\newcommand{\crcchange}[1]{#1}

\usepackage{enumitem}
\setlist{nosep}

\theoremstyle{definition}
\declaretheorem{definition}
\theoremstyle{plain}
\declaretheorem{lemma}
\declaretheorem{theorem}

\declaretheorem{statement}

\newcommand{\Abs}[1]{\left| #1 \right| }
\newcommand{\Prob}[2][]{\underset{#1}{\mathbf{P}}\left( \hiderel{#2} \right) }
\newcommand{\Probc}[2]{\mathbf{P}\left( \hiderel{#1} \middle | \hiderel{#2} \right) }

\newcommand{\norm}[1]{\left\| #1 \right\|}
\newcommand{\tvdist}[1]{\left\| #1 \right\|_{\mathrm{TV}}}

\title{Rapidly Mixing Gibbs Sampling for a Class of Factor Graphs Using Hierarchy Width}

\iftoggle{arxiv}{
  \author{
  Christopher De Sa$^\dagger$,
  Ce Zhang$^\ddagger$,
  Kunle Olukotun$^\dagger$, and
  Christopher R{\'e}$^\dagger$ \\
  $\dagger$ Stanford University,
  $\ddagger$ University of Wisconsin-Madison \\
  \texttt{cdesa@stanford.edu},
  \texttt{czhang@cs.wisc.edu}, \\
  \texttt{kunle@stanford.edu},
  \texttt{chrismre@stanford.edu}
  }
}
{
  \author{
  Christopher De Sa,
  Ce Zhang,
  Kunle Olukotun, and
  Christopher R{\'e} \\
  \texttt{cdesa@stanford.edu},
  \texttt{czhang@cs.wisc.edu}, \\
  \texttt{kunle@stanford.edu},
  \texttt{chrismre@stanford.edu} \\
  Departments of Electrical Engineering and Computer Science\\
  Stanford University, Stanford, CA 94309
  }
}

\usepackage{etoolbox}
\patchcmd{\thebibliography}{\section*{\refname}}{}{}{}

\begin{document}

\maketitle

\begin{abstract}
Gibbs sampling on factor graphs is a widely used inference technique,
which often produces good empirical results. Theoretical guarantees
for its performance are weak: even for tree structured graphs, the
mixing time of Gibbs may be exponential in the number of variables. To
help understand the behavior of Gibbs sampling, we introduce a new
(hyper)graph property, called \emph{hierarchy width}. We show that
under suitable conditions on the weights, bounded hierarchy width
ensures polynomial mixing time.  Our study of hierarchy width is in
part motivated by a class of factor graph templates,
\emph{hierarchical templates}, which have bounded hierarchy
width---regardless of the data used to instantiate them. We demonstrate
a rich application from natural language processing in which Gibbs
sampling provably mixes rapidly and achieves accuracy that exceeds
human volunteers.
\end{abstract}

\section{Introduction}
%% Factor graphs are commonly used in machine learning to represent
%% probabilistic models from a wide range of applications, including
%% probabilistic databases~\cite{wick2010scalable,jampani2008mcdb,
%%   jha2012probabilistic}, information
%% extraction~\cite{lafferty2001conditional,
%%   wang2008bayesstore,wang2011hybrid}, and insurance risk modeling
%% ~\cite{ antonio2007actuarial}.

We study inference on factor graphs using Gibbs sampling, the de facto
Markov Chain Monte Carlo (MCMC)
method~\cite[p. 505]{koller2009probabilistic}. Specifically,
our goal is to compute the
marginal distribution of some \emph{query} variables using Gibbs
sampling, given \emph{evidence} about some other variables and a set
of factor weights. We focus on the case where all variables are discrete. In
this situation, a Gibbs sampler randomly updates a single variable at
each iteration by sampling from its conditional distribution given the
values of all the other variables in the model.  Many systems---such
as Factorie~\cite{mccallum2009factorie}, 
OpenBugs~\cite{lunn2009bugs}, PGibbs~\cite{gonzalez2011parallel},
DimmWitted~\cite{zhang2014dimmwitted}, and
others~\cite{smola2010architecture,newman2007distributed,NIPS2012_4832}---use Gibbs
sampling for inference
because it is fast to run, simple to implement, and often
produces high quality empirical results. However, theoretical
guarantees about Gibbs are lacking. The aim of the technical result
of this paper is to
provide new cases in which one can guarantee that Gibbs gives
accurate results.

For an MCMC sampler like Gibbs sampling, the standard measure of
efficiency is the \emph{mixing time} of the underlying Markov chain.
% \begin{definition}[Mixing Time]
% The \emph{mixing time} of a Markov chain is the first time $t$ at
% which the estimated distribution $\mu_t$ is within statistical distance
% $\frac{1}{4}$
% of the true distribution~\cite[p. 55]{levin2009markov}.  That is,
% \[
%   \textstyle
%   t_{\mathrm{mix}}
%   =
%   \min \left\{
%     t
%     : 
%     \max_{A \subset \Omega} \Abs{\mu_t(A) - \pi(A)}
%     \le
%     \frac{1}{4}
%   \right\}.
% \] 
% \end{definition}
We say that a Gibbs sampler \emph{mixes rapidly} over a class of
models if its mixing time is at most polynomial in the number of
variables in the model.  Gibbs sampling is known to mix rapidly for
some models.  For example, Gibbs sampling on the Ising model on a
graph with bounded degree is known to mix in quasilinear time for high
temperatures~\cite[p. 201]{levin2009markov}.  \crcchange{Recent work has
outlined conditions under which Gibbs sampling of Markov Random Fields mixes
rapidly~\cite{NIPS2014_5315}.}  Continuous-valued Gibbs
sampling over models with exponential-family distributions is also
known to mix rapidly~\cite{diaconis2008,diaconis2010gibbs}. Each of
these celebrated
results still leaves a gap: there are many classes of factor graphs
on which Gibbs sampling seems to work very well---including as part of
systems that have won quality
competitions~\cite{surdeanu2014overview}---for which there are no
theoretical guarantees of rapid mixing.

Many graph algorithms that take exponential time in general can be
shown to run in polynomial time as long as some graph property is
bounded.  For inference on factor graphs, the most commonly used
property is hypertree width, which bounds the complexity of
\emph{dynamic programming} algorithms on the graph.  Many problems,
including variable elimination for exact inference, can be solved in
polynomial time on graphs with bounded hypertree
width~\cite[p. 1000]{koller2009probabilistic}.  In some sense, bounded
hypertree width is a necessary and sufficient condition for
tractability of inference in graphical models
\cite{kwisthout2010necessity,chandrasekaran2012complexity}.
Unfortunately, 
%Gibbs sampling is not based on use dynamic programming, and
it is not hard to construct examples of factor graphs with bounded
weights and hypertree width $1$ for which Gibbs sampling takes
exponential time to mix.  Therefore, bounding hypertree width is
insufficient to ensure rapid mixing of Gibbs sampling. To analyze the
behavior of Gibbs sampling, we define a new graph property, called the
\emph{hierarchy width}.  This is a stronger condition than hypertree
width; the hierarchy width of a graph will always be larger than its
hypertree width.  We show that for graphs with bounded hierarchy width
and bounded weights, Gibbs sampling mixes rapidly.

Our interest in hierarchy width is motivated by so-called factor graph
templates, which are common in
practice~\cite[p. 213]{koller2009probabilistic}.  Several types of
models, such as Markov Logic Networks (MLN) and Relational Markov
Networks (RMN) can be represented as factor graph templates.  Many
state-of-the-art systems use Gibbs sampling on factor graph templates
and achieve better results than competitors using other
algorithms~\cite{mccallum2009factorie,AAAI1510030}.  We exhibit a
class of factor graph templates, called \emph{hierarchical templates},
which, when instantiated, have a hierarchy width that is bounded
independently of the dataset used; Gibbs sampling on models
instantiated from these factor graph templates will mix in polynomial
time. This is a kind of sampling analog to tractable Markov
logic~\cite{domingos2012tractable} or so-called ``safe plans'' in probabilistic
databases~\cite{suciu2011probabilistic}.
We exhibit a real-world templated program
\crcchange{that outperforms} human annotators at a complex text extraction
task---and provably mixes in polynomial time. % (for any data source).

In summary, this work makes the following contributions:
\begin{itemize}
  %% \item We show by example that bounded hypertree width is insufficient
  %%   for rapid mixing of Gibbs sampling: there are classes of factor graph
  %%   templates (and therefore of factor graphs)
  %%   that mix in exponential time despite having hypertree width $1$ and bounded
  %%   factor weight.
  \item We introduce a new notion of width, \emph{hierarchy width}, 
    and show that Gibbs sampling mixes in polynomial time for all factor
    graphs with bounded hierarchy width and factor weight.
  \item We describe a new class of factor graph templates, \emph{hierarchical
    factor graph templates}, such that Gibbs sampling on instantiations of
    these templates mixes in polynomial time.
  \item We validate our results experimentally and exhibit factor
    graph templates that achieve high quality on tasks but for which
    our new theory is able to provide mixing time guarantees.
\end{itemize}

\subsection{Related Work}

Gibbs sampling is just one of several algorithms proposed for use in factor
graph inference.  The variable elimination algorithm
\cite{koller2009probabilistic} is an exact inference method that runs in
polynomial time for graphs of bounded hypertree width.  Belief propagation is 
another widely-used inference algorithm that produces an exact result for
trees and, although it does not converge in all cases, converges to a good
approximation under known conditions~\cite{ihler2005loopy}.
\emph{Lifted inference}
\cite{poole2003first} is one way to take advantage of the structural symmetry
of factor graphs that are instantiated from a template; there are lifted
versions of many common algorithms, such as variable elimination
\cite{ng2008probabilistic}, belief propagation \cite{singla2008lifted}, and
Gibbs sampling~\cite{liftedGibbs2012}.
It is also possible to leverage a template for fast computation:
~\citet{AAAI1510030} achieve orders of magnitude of speedup of Gibbs sampling
on MLNs.
Compared with Gibbs
sampling, these inference algorithms typically have better
theoretical results;
despite this, Gibbs sampling is a ubiquitous algorithm that performs
practically well---far outstripping its guarantees.

Our approach of characterizing runtime in terms of a graph property is
typical for the analysis of graph algorithms. 
Many algorithms are known to run in polynomial time on graphs of
bounded treewidth~\cite{robertson1986graph}, despite being otherwise NP-hard.
Sometimes, using a stronger or weaker property than treewidth will produce a
better result; for example,
the submodular width used for constraint satisfaction
problems~\cite{marx2013tractable}.

\section{Main Result}
In this section, we describe our main contribution.  We analyze some
simple example graphs, and use them to show that 
bounded hypertree width is not sufficient to guarantee rapid
mixing of Gibbs sampling.
Drawing intuition from this, we define the hierarchy width graph
property, and prove that Gibbs sampling mixes in polynomial time for graphs
with bounded hierarchy width.

First, we state some basic definitions.
A factor graph $G$ is a graphical model that consists of a set of variables $V$
and factors $\Phi$, and determines a distribution over those variables.
If $I$ is a \emph{world} for $G$ (an assignment of a
value to each variable in $V$), then $\epsilon$, the \emph{energy} of the
world, is defined as
\begin{equation}
  \label{eqnFactorGraphEnergy}
  \textstyle
  \epsilon(I)
  =
  \sum_{\phi \in \Phi} \phi(I).
\end{equation}
The probability of world $I$ is
$\pi(I) = \frac{1}{Z} \exp(\epsilon(I))$,
where $Z$ is the normalization constant necessary for this to be a 
distribution.
Typically, each $\phi$ depends only on a subset of the
variables; we can draw $G$ as a bipartite graph where a variable $v \in V$ is
connected to a factor $\phi \in \Phi$ if $\phi$ depends on $v$.
\begin{definition}[Mixing Time]
\crcchange{
The \emph{mixing time} of a Markov chain is the first time $t$ at
which the estimated distribution $\mu_t$ is within statistical distance
$\frac{1}{4}$
of the true distribution~\cite[p. 55]{levin2009markov}.  That is,}
\[
  \textstyle
  t_{\mathrm{mix}}
  =
  \min \left\{
    t
    : 
    \max_{A \subset \Omega} \Abs{\mu_t(A) - \pi(A)}
    \le
    \frac{1}{4}
  \right\}.
\]
\end{definition}
\subsection{Voting Example}
We start by considering a simple example model~\cite{wu2015incremental},
called the \emph{voting model}, that
models the sign of a particular ``query'' variable $Q \in \{ -1, 1 \}$ in the
presence of other ``voter'' variables
$T_i \in \{0, 1\}$ and $F_i \in \{0, 1\}$, for $i \in \{1, \ldots, n\}$,
that suggest that $Q$ is positive and negative (true and false), respectively.
We consider three versions of this model.  The first,
the \emph{voting model with linear semantics}, has energy function
\[
  \textstyle
  \epsilon(Q, T, F)
  =
  w Q \sum_{i=1}^n T_i
  -
  w Q \sum_{i=1}^n F_i
  +
  \sum_{i=1}^n w_{T_i} T_i
  +
  \sum_{i=1}^n w_{F_i} F_i,
\]
where \crcchange{$w_{T_i}$, $w_{F_i}$, and $w > 0$} are constant weights.
This model has a factor connecting each voter variable to the query, which
represents the value of that vote, and an additional factor that gives a prior
for each voter.  It corresponds to the factor graph in Figure
\ref{figVotingModelLinear}.
The second version, the \emph{voting model with logical semantics}, has
energy function
\[
  \textstyle
  \epsilon(Q, T, F)
  =
  w Q \max_{i} T_i
  -
  w Q \max_{i} F_i
  +
  \sum_{i=1}^n w_{T_i} T_i
  +
  \sum_{i=1}^n w_{F_i} F_i.
\]
Here, in addition to the prior factors, there are only two other factors, one
of which (which we call $\phi_T$)
connects all the true-voters to the query, and the other of which ($\phi_F$)
connects all the false-voters to the query.  The third version,
the \emph{voting model with ratio semantics}, is an intermediate between these
two models, and has energy function
\[
  \textstyle
  \epsilon(Q, T, F)
  =
  w Q \log\left(1 + \sum_{i=1}^n T_i \right)
  -
  w Q \log\left(1 + \sum_{i=1}^n F_i \right)
  +
  \sum_{i=1}^n w_{T_i} T_i
  +
  \sum_{i=1}^n w_{F_i} F_i.
\]
With either logical or ratio semantics, this model can be drawn as the
factor graph in Figure \ref{figVotingModelRatio}.

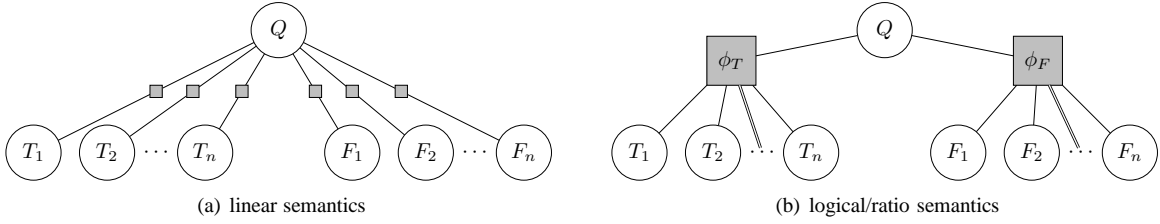
\begin{figure}[t]%
\centering
\subfigure[linear semantics]{%
\centering%
\resizebox{.44\textwidth}{!}{%
\begin{tikzpicture}[every node/.style={inner sep=0,outer sep=0}]
\draw (0,2) node[draw,circle,minimum size=0.9cm](q) {$Q$};
\draw (-4.0,0) node[draw,circle,minimum size=0.9cm](y1) {$T_1$};
\draw (-2.8,0) node[draw,circle,minimum size=0.9cm](y2) {$T_2$};
\draw (-2.0,0) node(ycd){$\cdots$};
\draw (-1.2,0) node[draw,circle,minimum size=0.9cm](y3) {$T_n$};
\draw (5.2-4.0,0) node[draw,circle,minimum size=0.9cm](n1) {$F_1$};
\draw (5.2-2.8,0) node[draw,circle,minimum size=0.9cm](n2) {$F_2$};
\draw (5.2-2.0,0) node(ncd){$\cdots$};
\draw (5.2-1.2,0) node[draw,circle,minimum size=0.9cm](n3) {$F_n$};
\draw (q) -- coordinate[midway](yf1) (y1);
\draw (q) -- coordinate[midway](yf2) (y2);
\draw (q) -- coordinate[midway](yf3) (y3);
\draw (q) -- coordinate[midway](nf1) (n1);
\draw (q) -- coordinate[midway](nf2) (n2);
\draw (q) -- coordinate[midway](nf3) (n3);
\draw (yf1) node[draw,fill=lightgray,rectangle,minimum size=0.2cm] {};
\draw (yf2) node[draw,fill=lightgray,rectangle,minimum size=0.2cm] {};
\draw (yf3) node[draw,fill=lightgray,rectangle,minimum size=0.2cm] {};
\draw (nf1) node[draw,fill=lightgray,rectangle,minimum size=0.2cm] {};
\draw (nf2) node[draw,fill=lightgray,rectangle,minimum size=0.2cm] {};
\draw (nf3) node[draw,fill=lightgray,rectangle,minimum size=0.2cm] {};
\end{tikzpicture}}
\label{figVotingModelLinear}}\qquad
\subfigure[logical/ratio semantics]{%
\centering%
\resizebox{.44\textwidth}{!}{%
\begin{tikzpicture}[every node/.style={inner sep=0,outer sep=0}]
\draw (0,2) node[draw,circle,minimum size=0.9cm](q) {$Q$};
\draw (-2.5,1.5) node[draw,fill=lightgray,rectangle,minimum size=0.8cm](py) {$\phi_T$};
\draw (2.5,1.5) node[draw,fill=lightgray,rectangle,minimum size=0.8cm](pn) {$\phi_F$};
\draw (-4.0,0) node[draw,circle,minimum size=0.9cm](y1) {$T_1$};
\draw (-2.8,0) node[draw,circle,minimum size=0.9cm](y2) {$T_2$};
\draw (-2.0,0) node(ycd){$\cdots$};
\draw (-1.2,0) node[draw,circle,minimum size=0.9cm](y3) {$T_n$};
\draw (5.2-4.0,0) node[draw,circle,minimum size=0.9cm](n1) {$F_1$};
\draw (5.2-2.8,0) node[draw,circle,minimum size=0.9cm](n2) {$F_2$};
\draw (5.2-2.0,0) node(ncd){$\cdots$};
\draw (5.2-1.2,0) node[draw,circle,minimum size=0.9cm](n3) {$F_n$};
\draw (q) -- (py);
\draw (q) -- (pn);
\draw (py) -- (y1);
\draw (py) -- (y2);
\draw (py) -- (y3);
\draw[double] (py) -- (ycd);
\draw (pn) -- (n1);
\draw (pn) -- (n2);
\draw (pn) -- (n3);
\draw[double] (pn) -- (ncd);
\end{tikzpicture}}
\label{figVotingModelRatio}}%
\caption{
Factor graph diagrams for the voting model; single-variable prior
factors are omitted.
}
\label{figVotingModel} 
\end{figure}

% The model has variable templates $Q$, $T$, and $F$, and factor templates
% The model contains two main factors: one that depends on $Q$ and the number of
% true $Y(i)$, and another that depends on $Q$ and the number of true $N(j)$;
% these factors condition $Q$ based on the amount of evidence for it.  The model
% also contains a factor for each $Y_i$ and $N_j$ that provides a prior for that
% variable.
% \begin{align*}
%   [w]&
%   \hspace{1em}
%   Q() \cdot T(x)
%   &
%   [-w]&
%   \hspace{1em}
%   Q() \cdot F(x)
%   \\
%   [w_{T(\hat x)}]&
%   \hspace{1em}
%   T(\hat x)
%   &
%   [w_{F(\hat x)}]&
%   \hspace{1em}
%   F(\hat x),
% \end{align*}
% where $w$, $w_{T(x)}$, and $w_{F(x)}$ are constant weights. 
% We call this the \emph{voting model}:
% here, the top two factor templates represent variables $T(x)$ and $F(x)$ voting
% about the sign of $Q$ while the bottom two give priors on the variables.
% Assume that we instantiate this model with $x \in \{1 \ldots n \}$.
% From (\ref{eqnEnergySemantic}),
% if we represent a world as a triple $I = (q, t, f)$, where $q \in \{-1, 1\}$,
% $t \in \{0, 1\}^n$, and $f \in \{0, 1\}^n$, this model has energy function
% \[
%   \epsilon(q, t, f)
%   =
%   w q \, g(\mathbf{1}^T t)
%   -
%   w q \, g(\mathbf{1}^T f)
%   +
%   \sum_{x=1}^n w_{T(x)} t_x
%   +
%   \sum_{x=1}^n w_{F(x)} f_x,
% \]
% where $\mathbf{1}$ is the vector of all $1$s. Notice that
% the weight function depends on the semantic function $g$.

\crcchange{These three cases model different distributions and therefore
different ways of representing the power of a vote;
the choice of names is motivated by} considering the marginal
odds of $Q$ given the other variables.
For linear semantics, the odds of $Q$ depend \emph{linearly} on the difference
between
the number of nonzero positive-voters $T_i$ and nonzero negative-voters $F_i$.
For ratio semantics,
the odds of $Q$ depend roughly on their \emph{ratio}.  For logical semantics,
only the presence of nonzero voters matters, not the number of voters.

% Figure \ref{figVotingModelRatio} is a diagram of this factor graph.  For linear
% semantics in particular, we can decompose the energy function further as
% \[
%   \epsilon(q, t, f)
%   =
%   \sum_{x=1}^n w q t_x
%   -
%   \sum_{x=1}^n w q f_x
%   +
%   \sum_{x=1}^n w_{T(x)} t_x
%   +
%   \sum_{x=1}^n w_{F(x)} f_x;
% \]
% this results in the factor graph in Figure \ref{figVotingModelLinear}.

We instantiated this model with random weights $w_{T_i}$ and $w_{F_i}$, ran
Gibbs sampling on it, and computed the variance of 
the estimated marginal probability of $Q$ for the different models
(Figure \ref{figVotingPlot}).  The results show that the models with
logical and ratio semantics produce much lower-variance estimates than the
model with linear semantics.
This experiment motivates us to try to prove a bound on the mixing time of
Gibbs sampling on this model.

% \begin{figure}[h]%
% \centering
% \resizebox{!}{.30\textwidth}{%
% \Large \input{plotvoting.tex}%
% }%
% \caption{Convergence for the voting model with $n = 50$,
% $w = 0.5$, and $\{ w_{T_x}, w_{F_x} \} \subset (-1, 0)$.}
% %, and random negative priors greater than $-1$.}%
% \label{figVotingPlot}%
% \end{figure}

\begin{figure}[h]
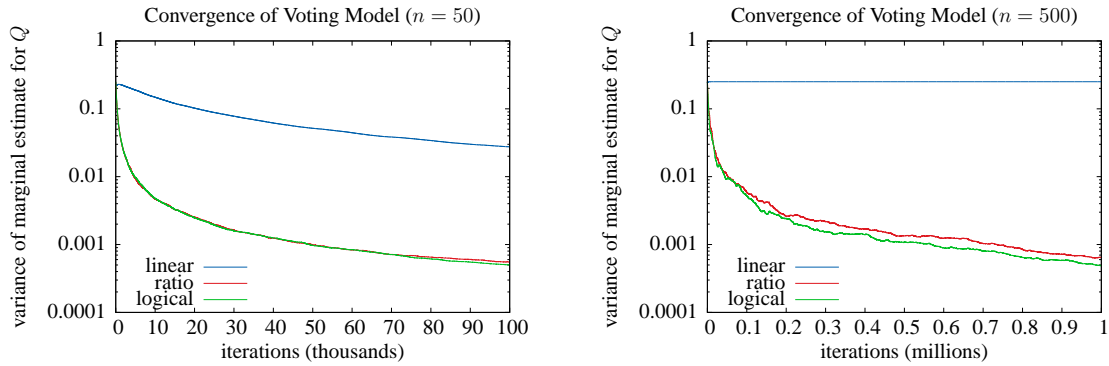
%
\centering
\subfigure{%
\centering%
\resizebox{!}{.30\textwidth}{%\resizebox{!}{.30\textwidth}{%
\Large \input{plotvoting.tex}%
}
\label{figVoting50Plot}}\qquad%
\subfigure{%
\centering%
\resizebox{!}{.30\textwidth}{%\resizebox{!}{.30\textwidth}{%
\Large \input{plotvoting500.tex}%
}
\label{figVoting500Plot}}\qquad%
\caption{
Convergence for the voting model with
$w = 0.5$, and random prior weights in $(-1, 0)$.
}
\label{figVotingPlot} 
\end{figure}

\begin{theorem}
\label{thmVotingBounds}
\crcchange{
Fix any constant $\omega > 0$, and
run Gibbs sampling on the voting model
with bounded factor weights
$\{w_{T_i}, w_{F_i}, w \} \subset [-\omega, \omega]$.}
For the voting model with linear semantics, the largest possible mixing
time $t_{\mathrm{mix}}$ of
any such model is $t_{\mathrm{mix}} = 2^\Theta(n)$.  For the voting model 
with either logical or
ratio semantics, the largest possible mixing time is
$t_{\mathrm{mix}} = \Theta(n \log n)$.
\end{theorem}

% \begin{table}[h]%
% \caption{Worst-case mixing time bounds for voting model with $n$ objects.}%
% \label{tabVotingBounds}%
% \begin{center}
% \begin{tabu}{r|[2pt]c|c|c}
% Semantic & Logical & Ratio & Linear \\
% \tabucline[2pt]{-}
% Upper Bound & $O(n \log n)$ & $O(n \log n)$ & $2^{O(n)}$ \\
% \hline
% Lower Bound & $\Omega(n \log n)$ & $\Omega(n \log n)$ 
%   & $2^{\Omega(n)}$ \\
% \end{tabu}
% \end{center}
% \end{table}

This result validates our observation that linear semantics mix poorly
compared to logical and ratio semantics.  Intuitively, the reason why linear
semantics performs worse is that the Gibbs sampler will switch
the state of $Q$ only very infrequently---in fact exponentially so.
This is because the energy roughly depends linearly on the number of voters
$n$, and therefore the probability of switching $Q$ depends exponentially on
$n$.  This does not happen in either the logical or ratio models.

\subsection{Hypertree Width}
\label{ssHypertreeWidth}
In this section, we describe the commonly-used graph property of hypertree
width, and show using the voting example that bounding it is insufficient to
ensure rapid Gibbs sampling.  
Hypertree width is typically used to bound the complexity of
dynamic programming algorithms on a graph; in particular,
variable elimination for
exact inference runs in polynomial time on factor graphs with bounded
hypertree width~\cite[p. 1000]{koller2009probabilistic}.
The hypertree width of a hypergraph, which we
denote $\mathsf{tw}(G)$, is a
generalization of the notion of acyclicity; since the definition of hypertree
width is
technical, we instead state the definition of an acyclic hypergraph, which
is sufficient for our analysis.  
In order to apply these notions to factor graphs, we can represent a factor
graph as a hypergraph that has one vertex for each node of the factor graph,
and one hyperedge for each factor, where that hyperedge contains all
variables the factor depends on.

\begin{definition}[Acyclic Factor Graph~\cite{gottlob2014treewidth}]
A \emph{join tree}, also called a junction tree,
of a factor graph $G$ is a tree $T$ such that the nodes of $T$ are the factors
of $G$ and, if two factors $\phi$ and $\rho$ both depend on the same
variable $x$ in $G$, then every factor on the unique path between $\phi$ 
and $\rho$ in $T$ also depends on $x$.
A factor graph is \emph{acyclic} if it has a join tree.  All acyclic graphs
have hypertree width $\mathsf{tw}(G) = 1$.
\end{definition}
Note that all trees are acyclic; in particular the voting model (with any
semantics) has hypertree width $1$.  Since the voting model with
linear semantics and bounded weights mixes in exponential time
(Theorem \ref{thmVotingBounds}), this means that bounding the hypertree width
and the factor weights is insufficient to ensure rapid mixing of Gibbs
sampling.

\subsection{Hierarchy Width}
\label{secHierarchyWidth}
Since the hypertree width is insufficient, we define a new graph
property, the \emph{hierarchy width}, which, when bounded, ensures rapid mixing
of Gibbs sampling.  This result is our main contribution.

\begin{definition}[Hierarchy Width]
The hierarchy width $\mathsf{hw}(G)$ of a factor graph $G$ is defined
recursively such that, for any \emph{connected}
factor graph $G = \langle V, \Phi \rangle$,
\begin{equation}
  \label{eqnHWConnected}
  \mathsf{hw}(G)
  =
  1 + \min_{\phi^* \in \Phi}
  \mathsf{hw}(\langle V, \Phi - \{ \phi^* \}\rangle),
\end{equation}
and for any \emph{disconnected} factor graph $G$ with connected components
$G_1, G_2, \ldots$,
\begin{equation}
  \label{eqnHWDisconnected}
  \mathsf{hw}(G) = \max_i \mathsf{hw}(G_i).
\end{equation}
As a \emph{base case}, all factor graphs $G$ with no factors have
\begin{equation}
  \label{eqnHWBase}
  \mathsf{hw}(\langle V, \emptyset \rangle) = 0.
\end{equation}
\end{definition}

To develop some intuition about how to use the definition of
hierarchy width, we derive the hierarchy width of
the path graph drawn in Figure \ref{figPathGraphDiag}.

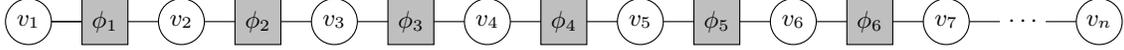
\begin{figure}[h]%
\centering
\resizebox{.90\textwidth}{!}{%
\begin{tikzpicture}[every node/.style={inner sep=0,outer sep=0}]
\draw (0.0,0) node[draw,circle,minimum size=0.6cm] (v1) {\small $v_1$};
\draw (1.0,0) node[draw,fill=lightgray,rectangle,minimum size=0.6cm] (f1) {\small $\phi_1$};
\draw (2.0,0) node[draw,circle,minimum size=0.6cm] (v2) {\small $v_2$};
\draw (3.0,0) node[draw,fill=lightgray,rectangle,minimum size=0.6cm] (f2) {\small $\phi_2$};
\draw (4.0,0) node[draw,circle,minimum size=0.6cm] (v3) {\small $v_3$};
\draw (5.0,0) node[draw,fill=lightgray,rectangle,minimum size=0.6cm] (f3) {\small $\phi_3$};
\draw (6.0,0) node[draw,circle,minimum size=0.6cm] (v4) {\small $v_4$};
\draw (7.0,0) node[draw,fill=lightgray,rectangle,minimum size=0.6cm] (f4) {\small $\phi_4$};
\draw (8.0,0) node[draw,circle,minimum size=0.6cm] (v5) {\small $v_5$};
\draw (9.0,0) node[draw,fill=lightgray,rectangle,minimum size=0.6cm] (f5) {\small $\phi_5$};
\draw (10.0,0) node[draw,circle,minimum size=0.6cm] (v6) {\small $v_6$};
\draw (11.0,0) node[draw,fill=lightgray,rectangle,minimum size=0.6cm] (f6) {\small $\phi_6$};
\draw (12.0,0) node[draw,circle,minimum size=0.6cm] (v7) {\small $v_7$};
\draw (13.0,0) node[minimum size=0.6cm] (f7) {\small $\cdots$};
\draw (14.0,0) node[draw,circle,minimum size=0.6cm] (v8) {\small $v_n$};
\draw (f1) -- (v1);
\draw (v1) -- (f1);
\draw (f1) -- (v2);
\draw (v2) -- (f2);
\draw (f2) -- (v3);
\draw (v3) -- (f3);
\draw (f3) -- (v4);
\draw (v4) -- (f4);
\draw (f4) -- (v5);
\draw (v5) -- (f5);
\draw (f5) -- (v6);
\draw (v6) -- (f6);
\draw (f6) -- (v7);
\draw (v7) -- (f7);
\draw (f7) -- (v8);
\end{tikzpicture}}
\caption{
Factor graph diagram for an $n$-variable path graph.
}
\label{figPathGraphDiag} 
\end{figure}

\begin{lemma}
The path graph model has hierarchy width
$\mathsf{hw}(G) = \lceil \log_2 n \rceil$.
\end{lemma}
\begin{proof}
Let $G_n$ denote the path graph with $n$ variables.  For $n = 1$, the
lemma follows from (\ref{eqnHWBase}). For $n > 1$, $G_n$ is connected, so
we must compute its hierarchy width by applying
(\ref{eqnHWConnected}).  It turns out that the factor that minimizes this
expression is the factor in the middle, and so applying
(\ref{eqnHWConnected}) followed by (\ref{eqnHWDisconnected}) shows that
$\mathsf{hw}(G_n) = 1 + \mathsf{hw(G_{\lceil \frac{n}{2} \rceil})}$.
Applying this inductively proves the lemma.
\end{proof}

Similarly, we are able to compute the hierarchy
width of the voting model factor graphs.

\begin{lemma}
The voting model with logical or ratio semantics has hierarchy width
$\mathsf{hw}(G) = 3$.
\end{lemma}
% \begin{proof}
% This model has a hypergraph representation $G$ that is connected, so we apply
% (\ref{eqnHWConnected}) to remove one of the factors.  Let's say that we remove
% factor $\phi_Y$, to produce a new graph $G^{(1)}$.  From (\ref{eqnHWConnected}),
% we have that $\mathsf{hw}(G) \le 1 + \mathsf{hw}(G^{(1)})$.  Now, $G^{(1)}$
% is disconnected (since the $Y_i$ are no longer connected to $Q$ by $\phi_Y$);
% therefore we apply (\ref{eqnHWDisconnected}) to split the graph into
% $a + 1$ subgraphs: $G^{(2)}$, which contains $Q$ and the $N_i$, and
% $G^{(Y_i)}$, which contains just $Y_i$.  From (\ref{eqnHWDisconnected}),
% \[
%   \mathsf{hw}(G^{(1)})
%   = 
%   \max\left(\mathsf{hw}(G^{(2)}), \max_i \mathsf{hw}(G^{(Y_i)}) \right). 
% \]
% Now, $G^{(2)}$ is connected, so we once again can apply (\ref{eqnHWConnected}),
% removing $\phi_N$ to produce graph $G^{(3)}$, where
% $\mathsf{hw}(G^{(2)}) \le 1 + \mathsf{hw}(G^{(3)})$.  As above, $G^{(3)}$ is 
% disconnected, and so by (\ref{eqnHWDisconnected}), we can split it into
% \[
%   \mathsf{hw}(G^{(3)})
%   = 
%   \max\left(\mathsf{hw}(G^{(Q)}), \max_j \mathsf{hw}(G^{(N_j)}) \right). 
% \]
% Now, each $G^{(Y_i)}$ and $G^{(N_j)}$ has only one vertex and one factor, so
% we can apply (\ref{eqnHWConnected}) followed by (\ref{eqnHWBase}) to show that
% $\mathsf{hw}(G^{(Y_i)}) = \mathsf{hw}(G^{(N_i)}) = 1$.  Similarly, we can apply
% (\ref{eqnHWBase}) to show that $\mathsf{hw}(G^{(Q)}) = 0$.  Combining all our
% equations  produces the desired result.
% \end{proof}

\begin{lemma}
The voting model with linear semantics has hierarchy width
$\mathsf{hw}(G) = 2n + 1$.
\end{lemma}

These results are promising, since they separate our polynomially-mixing
examples from our exponentially-mixing examples.  However, the hierarchy
width of a factor graph says nothing about the factors themselves and the
functions they compute.  This means that it, alone, tells us nothing about
the model; for example, any distribution can be represented by a trivial
factor graph with a single factor that contains all the variables.
Therefore, in order to use hierarchy width to produce a result
about the mixing time of Gibbs sampling, we constrain the maximum weight of
the factors.

% Unfortunately, we could have written the voting model with
% linear semantics without decomposing the factors, resulting in a model that
% looks like Figure \ref{figVotingModelRatio}; in this case, its hierarchy width
% would be $3$, but it would still mix in exponential time.  In general, by
% combining multiple factors into one, we can reduce the hierarchy width without
% changing the behavior of the model.  This suggests that,
% in order to mix rapidly, we must also constrain the maximum weight of the
% factors.

\begin{definition}[Maximum Factor Weight]
A factor graph has maximum factor weight $M$, where
\[
  M
  =
  \max_{\phi \in \Phi}
  \left(
    \max_I \phi(I)
    -
    \min_I \phi(I)
  \right).
\]
\end{definition}
For example, the maximum factor weight of the voting example with linear
semantics is $M = 2w$; with logical semantics, it is $M = 2w$; and with
ratio semantics, it is $M = 2w \log(n+1)$.  We now show that
graphs with bounded hierarchy width and maximum factor weight mix
rapidly.

\begin{restatable}[Polynomial Mixing Time]{theorem}{thmPolynomialMixingTime}
\label{thmPolynomialMixingTime}
If $G$ is a factor graph with $n$ variables, at most $s$ states per variable,
$e$ factors, maximum factor weight $M$, and hierarchy width $h$, then
\[
  t_{\mathrm{mix}}
  \le
  \left(
    \log(4)
    +
    n \log(s)
    +
    e M
  \right)
  n \exp(3hM).
\]
In particular, if $e$ is polynomial in $n$, the number
of values for each variable is bounded, and $hM = O(\log n)$, then
$t_{\mathrm{mix}}(\epsilon) = O(n^{O(1)})$.
\end{restatable}

To show why bounding the hierarchy width is necessary for this result,
we outline the proof of Theorem \ref{thmPolynomialMixingTime}.
Our technique involves bounding the absolute spectral gap $\gamma(G)$ of the
transition matrix of Gibbs sampling on graph $G$; there are standard
results that use the absolute spectral gap to bound the mixing time of
a process~\cite[p. 155]{levin2009markov}.  Our proof proceeds via induction
using the definition of hierarchy width and the following three lemmas.

\begin{restatable}[Connected Case]{lemma}{lemmaRemoveOneFactor}
\label{lemmaRemoveOneFactor}
Let $G$ and $\bar G$ be two factor graphs with maximum factor weight $M$,
which differ only inasmuch as $G$ contains a single additional factor $\phi^*$.
Then,
\[
  \gamma(G) \ge \gamma(\bar G) \exp\left(-3M \right).
\]
\end{restatable}

\begin{restatable}[Disconnected Case]{lemma}{lemmaReduceIndependent}
\label{lemmaReduceIndependent}
Let $G$ be a disconnected factor graph with $n$ variables and $m$ connected
components $G_1, G_2, \ldots, G_m$ with $n_1, n_2, \ldots n_m$ variables,
respectively. Then,
\[
  \gamma(G)
  \ge
  \min_{i \le m} \frac{n_i}{n} \gamma(G_i).
\]
\end{restatable}

\begin{restatable}[Base Case]{lemma}{lemmaNoFactors}
\label{lemmaNoFactors}
Let $G$ be a factor graph with one variable and no factors.  The absolute
spectral gap of Gibbs sampling running on $G$ will be
$\gamma(G) = 1$.
\end{restatable}

Using these Lemmas inductively, it is
not hard to show that, under the conditions of Theorem
\ref{thmPolynomialMixingTime},
\[
  \gamma(G)
  \ge
  \frac{1}{n}
  \exp\left( -3 h M \right);
\]
converting this to a bound on the mixing time produces the result of
Theorem \ref{thmPolynomialMixingTime}.

% This is a generalization of our previous result, because
% this theorem, along with Lemma \ref{lemmaHierarchicalHW}, immediately
% provides a proof of Theorem \ref{thmHierarchicalFG}.

% This result allows us to relate the hierarchy width and the hypertree width
% of a hypergraph.
To gain more intuition about the hierarchy width, we compare its properties
to
those of the hypertree width.
\crcchange{First, we note that, when the hierarchy width is
bounded, the hypertree width is also bounded.}

\begin{restatable}{statement}{stmtHierarchyWidthComparison}
\label{stmtHierarchyWidthComparison}
For any factor graph $G$, $\mathsf{tw}(G) \le \mathsf{hw}(G)$.
\end{restatable}
One of the useful properties of the hypertree width is that, for any fixed $k$,
computing whether a graph $G$ has hypertree width $\mathsf{tw}(G) \le k$ can
be done in polynomial time in the size of $G$.  We show the same is true for
the hierarchy width.
\begin{restatable}{statement}{stmtHierarchyWidthKPoly}
\label{stmtHierarchyWidthKPoly}
For any fixed $k$, computing whether $\mathsf{hw}(G) \le k$ can be done
in time polynomial in the number of factors of $G$.
\end{restatable}
Finally, we note that we can also bound the hierarchy width using the degree
of the factor graph.
Notice that a graph with unbounded node degree
contains the voting program with linear semantics as a subgraph.  This
statement shows that bounding the hierarchy width disallows such graphs.
\begin{restatable}{statement}{stmtHierarchyWidthDegree}
\label{stmtHierarchyWidthDegree}
\crcchange{Let $d$ be the maximum degree of a variable in factor graph $G$.
Then, $\mathsf{hw}(G) \ge d$.}
\end{restatable}

\section{Factor Graph Templates}
\label{ssFactorGraphTemplates}
% A factor graph $G$ is a graphical model that consists of a set of variables $V$
% and factors $\Gamma$, and determines a distribution over those variables.
% Specifically, if $I$ is a \emph{world} for $G$ (an assignment of a
% value to each variable in $V$), then $\epsilon$, the \emph{energy} of the
% world, is defined as
% \[
%   \epsilon(I)
%   =
%   \sum_{\gamma \in \Gamma} w_\gamma \gamma(I),
% \]
% where $w_\gamma$ is the weight of factor $\gamma$.
% The probability of world $I$ is then defined as
% \[
%   \pi(I) = \frac{1}{Z} \exp(\epsilon(I)),
% \]
% where $Z$ is the normalization constant necessary for this to be a 
% distribution.
% Typically, each $\gamma$ depends only on some subset of the
% variables; we can draw $G$ as a bipartite graph where a variable $v \in V$ is
% connected to a factor $\gamma \in \Gamma$ if $\gamma$ depends on $v$.
Our study of hierarchy width is in part motivated by the desire to analyze
the behavior of Gibbs sampling on
factor graph templates, which are common in practice and used by many
state-of-the-art systems.  A factor graph template is an
abstract model that can be
\emph{instantiated} on a dataset to produce a factor graph.
The dataset consists of
\emph{objects}, each of which represents a thing we want to reason about, which
are divided into \emph{classes}.  For example, the object
$\mathsf{Bart}$ could have class $\mathsf{Person}$
and the object $\mathsf{Twilight}$ could have class $\mathsf{Movie}$.
(There are many ways to define templates; here, we follow the formulation
in \citet[p. 213]{koller2009probabilistic}.)

A factor graph template consists of a set of template variables and template
factors.  A template variable represents a property of a tuple of zero or more
objects of particular classes.  For example, we could have an
$\mathsf{IsPopular}(x)$ template, which takes a single argument of class
$\mathsf{Movie}$.  In the instantiated graph, this would take the form of 
multiple variables like $\mathsf{IsPopular}(\mathsf{Twilight})$ or
$\mathsf{IsPopular}(\mathsf{Avengers})$.  Template factors are replicated
similarly to produce multiple factors in the instantiated graph.  For example,
we can have a template factor
\begin{align*}
  \phi\left(\mathsf{TweetedAbout}(x, y), \mathsf{IsPopular}(x) \right)
\end{align*}
for some factor function $\phi$. This would be instantiated to factors like
\begin{align*}
  \phi\left(
    \mathsf{TweetedAbout}(\mathsf{Avengers}, \mathsf{Bart}), 
    \mathsf{IsPopular}(\mathsf{Avengers})
  \right).
  % \\
  % \phi\left(
  %   \mathsf{TweetedAbout}(\mathsf{Avengers}, \mathsf{Lisa}), 
  %   \mathsf{IsPopular}(\mathsf{Avengers})
  % \right).
\end{align*}
We call the $x$ and $y$ in a template factor \emph{object symbols}.
For an instantiated factor graph with template factors $\Phi$, 
if we let $A_{\phi}$ denote the set of possible assignments to the object
symbols in a template factor $\phi$, and let $\phi(a, I)$ denote the value of
its factor function in world $I$ under the object symbol assignment $a$, then
the standard way to define the energy function is with
\begin{equation}
  \label{eqnEnergyStandard}
  \textstyle
  \epsilon(I)
  =
  \sum_{\phi \in \Phi}
  \sum_{a \in A_{\phi}}
  w_\phi
  \phi(a, I),
\end{equation}
where $w_\phi$ is the weight of template factor $\phi$.
This energy function results from the creation of a single factor
$\phi_a(I) = \phi(a, I)$ for each object symbol assignment $a$ of $\phi$.
Unfortunately, this standard energy definition is not suitable for all
applications.  To deal with this, \citet{wu2015incremental} introduce
the notion of a \emph{semantic function} $g$, which counts the of energy of
instances of the factor template in a non-standard way.  In order to do this,
they first divide the object symbols of each template factor into two groups,
the \emph{head symbols} and the \emph{body symbols}.  When writing out
factor templates, we distinguish head symbols by writing them with a hat
(like $\hat x$).  If we let $H_\phi$
denote the set of possible assignments to the head symbols, let $B_\phi$ denote
the set of possible assignments to the body symbols, and let $\phi(h, b, I)$
denote the value of its factor function in world $I$ under the assignment
$(h, b)$, then the energy of a world is defined as
\begin{equation}
  \label{eqnEnergySemantic}
  \textstyle
  \epsilon(I)
  =
  \sum_{\phi \in \Phi}
  \sum_{h \in H_\phi}
  w_\phi(h)
  \>
  g\left(
    \sum_{b \in B_\phi}
    \phi(h, b, I)
  \right).
\end{equation}
This results in the creation of a single factor
$\phi_h(I) = g \left( \sum_b \phi(h, b, I) \right)$ for each assignment of
the template's head symbols.
We focus on three semantic functions in
particular~\cite{wu2015incremental}.
For the first, \emph{linear semantics},
$g(x) = x$.
This is identical to the standard semantics in (\ref{eqnEnergyStandard}).
For the second,
\emph{logical semantics}, $g(x) = \mathrm{sgn}(x)$.
For the third, \emph{ratio semantics},
$g(x) = \mathrm{sgn}(x) \log(1 + \Abs{x})$.  These semantics
are analogous to the different semantics used in our voting example.
\crcchange{\citet{wu2015incremental} exhibit several classification problems
where using logical or ratio semantics gives better F1 scores.}

\begin{figure}[t]%
\centering
\resizebox{.5\textwidth}{!}{%
\begin{tikzpicture}
\draw (-5,-2) rectangle (5,2) node[below left]{\it bounded factor weight};
\clip (-5.25,-2.25) rectangle (5.02,2.02);
\draw[rounded corners=1cm] (-7.5,-4.5) rectangle (4.5,1.5);
\draw (-0.5,1.5) node[below]{\it bounded hypertree width};
\fill (-4,-0.5) circle (0.05) node[above,style={align=center}] {voting\\(linear)};
\draw[rounded corners=0.8cm] (-3.0,-4.5) rectangle (4.0,1.0);
\draw (0.5,1.0) node[below]{\it polynomial mixing time};
\draw[rounded corners=0.6cm] (-2.5,-4.0) rectangle (3.5,0.5);
\draw (0.5,0.5) node[below]{\it bounded hierarchy width};
\draw[rounded corners=0.4cm] (-2.0,-1.75) rectangle (3.0,0.0);
\draw (0.5,0.0) node[below]{\it hierarchical templates};
\fill (-1,-1.4) circle (0.05) node[above,style={align=center}] {voting\\(logical)};
\fill (2,-1.4) circle (0.05) node[above,style={align=center}] {voting\\(ratio)};
\end{tikzpicture}}
\vspace{-5pt}
\caption{Subset relationships among classes of factor graphs, and locations of examples.}%
\label{figGraphClasses}%
\end{figure}
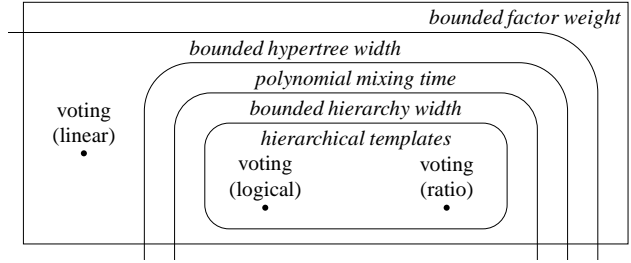

\subsection{Hierarchical Factor Graphs}
\label{ssHierarchicalFactorGraphs}
In this section, we outline a class of templates, hierarchical templates,
that have bounded hierarchy width.
% We seek to generalize the results from the voting model to more complicated
% factor graphs.  We know that bounded hypertree width isn't enough, so we
% look for other structure in the model.
% Intuitively, we see the voting model as having a
% hierarchical structure: the voters $T(x)$ and $F(x)$ are subordinate to the
% query $Q$, and otherwise independent from each other, which allows for rapid
% convergence; this is a similar property to that used by
% \citet{domingos2012tractable} to describe a tractable subset of MLNs.
\crcchange{We focus on models that
have hierarchical structure} in their template factors; for example,
\begin{equation}
  \label{eqnVeryHierarchical}
  \phi( A(\hat x, \hat y, z) , B(\hat x, \hat y) , Q(\hat x, \hat y) )
\end{equation}
should have hierarchical structure, while
\begin{equation}
  \label{eqnNotHierarchical}
  \phi( A(z) , B(\hat x) , Q(\hat x, y) )
\end{equation}
should not.  Armed with this intuition, we give the following definitions.

\begin{definition}[Hierarchy Depth]
A template factor $\phi$ has \emph{hierarchy depth} $d$ if the first $d$
object symbols
that appear in each of its terms are the same.  We call these symbols
\emph{hierarchical symbols}.  For example, (\ref{eqnVeryHierarchical})
has hierarchy depth $2$, and $\hat x$ and $\hat y$ are hierarchical symbols;
also, (\ref{eqnNotHierarchical}) has hierarchy depth $0$, and no hierarchical
symbols.
\end{definition}

\begin{definition}[Hierarchical]
We say that a template factor is \emph{hierarchical} if all of its
head symbols are hierarchical symbols.  For example,
(\ref{eqnVeryHierarchical}) is hierarchical, while
(\ref{eqnNotHierarchical}) is not.  We say that a factor graph template is
\emph{hierarchical} if all its template factors are hierarchical.
\end{definition}

We can explicitly bound the hierarchy width of instances of 
hierarchical factor graphs.
\begin{restatable}{lemma}{lemmaHierarchicalHW}
\label{lemmaHierarchicalHW}
If $G$ is an instance of a hierarchical template
with $E$ template factors, then $\mathsf{hw}(G) \le E$.
\end{restatable}

We would now like to use Theorem \ref{thmPolynomialMixingTime} to prove a bound
on the mixing time; this requires us to bound the maximum factor weight of
the graph.
Unfortunately, for linear semantics, the maximum factor weight of a graph
is potentially $O(n)$, so applying Theorem \ref{thmPolynomialMixingTime}
won't get us useful results.
Fortunately, for logical or ratio semantics, hierarchical factor graphs do
mix in polynomial time.
\begin{restatable}{statement}{stmtHierarchicalFG}
\label{stmtHierarchicalFG}
For any fixed hierarchical factor graph template $\mathcal{G}$,
if $G$ is an instance of $\mathcal{G}$
with bounded weights using either logical or ratio semantics, then the
mixing time of Gibbs sampling on $G$ is polynomial in the
number of objects $n$ in
its dataset.  That is, $t_{\mathrm{mix}} = O\left(n^{O(1)} \right)$.
\end{restatable}
So, if we want to construct models with Gibbs samplers that mix rapidly, one
way to do it is with hierarchical factor graph templates using logical
or ratio semantics.

\section{Experiments}

% \begin{figure}[h]%
% \centering
% \resizebox{!}{.30\textwidth}{%
% \Large \input{plotvariance.tex}%
% }%
% \caption{Variance of marginal estimates produced by Gibbs sampling running on
% the KBP dataset after some number of samples of all the variables. 
% Variance was computed empirically over 20 samples, each produced by
% an independent run of the DimmWitted Gibbs sampler~\cite{zhang2014dimmwitted}.}
% \label{figVariancePlot}%
% \end{figure}

\begin{figure}[t]
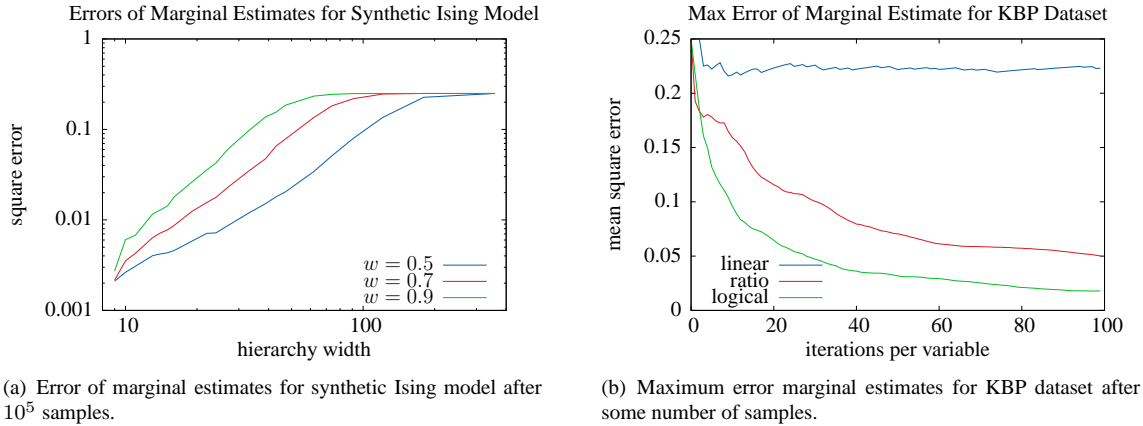
%
\centering
\subfigure[Error of marginal estimates for synthetic Ising model
after $10^5$ samples.]{%
\centering%
\resizebox{!}{.30\textwidth}{%\resizebox{!}{.30\textwidth}{%
\Large \input{plotsynth.tex}%
}
\label{figSynthPlot}}
\qquad
\subfigure[Maximum error marginal estimates for KBP dataset after some number
of samples.]{%
\centering%
\resizebox{!}{.30\textwidth}{%\resizebox{!}{.30\textwidth}{%
\Large \input{plotvariance.tex}%
}
\label{figVariancePlot}}
\caption{
Experiments illustrate how convergence is affected by hierarchy width and
semantics.
}
\label{figExperiments} 
\end{figure}

% \subfigure[Average coupling time (samples per variable) of two random
% Gibbs sampling chains.]{%
% \centering%
% \begin{minipage}[b][0.30\textwidth][c]{.40\textwidth}
% \centering
% \begin{tabu}{r|[2pt]c|c}
% Dataset & Ratio & Linear \\
% \tabucline[2pt]{-}
% KBP & 2.35 & 8.75 \\
% \hline
% Genomics & 2.0 & 75.4 \\
% \end{tabu}
% \end{minipage}
% \label{figCoupling}}

% We validate our theoretical result on both synthetic data and real-world
% applications.

\paragraph*{Synthetic Data} We constructed a synthetic dataset by using
an ensemble of Ising model graphs each with $360$
nodes, $359$ edges, and treewidth $1$, but with different hierarchy widths.
These graphs ranged from the star graph
(like in Figure \ref{figVotingModelLinear}) to the path graph; and each
had different hierarchy width.
For each graph, we were able to calculate the exact true marginal
of each variable because of the small tree-width. We then ran Gibbs sampling 
on each graph, and calculated the error of the marginal estimate of a 
single arbitrarily-chosen query variable.
Figure~\ref{figSynthPlot} shows the result with different weights and 
hierarchy width. It shows that, even for tree graphs with the same
number of nodes and edges, the mixing time can still vary depending on the
hierarchy width of the model.

\paragraph*{Real-World Applications} We observed that the hierarchical templates
that we focus on in this work appear frequently in real applications. For example,
all five knowledge base population (KBP) systems illustrated by 
Shin et al.~\cite{wu2015incremental} contain subgraphs that are grounded by hierarchical 
templates. Moreover, sometimes a factor graph is solely grounded by
hierarchical templates, and thus provably mixes rapidly by our
theorem while achieving high quality. To validate this, we constructed a
hierarchical template for the Paleontology application used by
Shanan et al.~\cite{Peters:2014:ArXiv}.
We found that when using the ratio semantic, we were able to get an F1 score
of 0.86 with precision of 0.96. On the same task, this quality is actually higher
than professional human volunteers~\cite{Peters:2014:ArXiv}. For comparison,
the linear semantic achieved an F1 score of 0.76 and the logical achieved 0.73.

The factor graph we used in this Paleontology application is large enough that
it is intractable, \crcchange{using exact inference, to}
estimate the true marginal to investigate the mixing behavior.
Therefore, we chose a subgraph of a KBP system used by
Shin et al.~\cite{wu2015incremental} that can be grounded by a
hierarchical template
and chose a setting of the weight such that the true marginal was 0.5 for
all variables. We then ran Gibbs sampling on this subgraph and report
the average error of the marginal estimation in Figure~\ref{figVariancePlot}.
Our results illustrate the effect of changing the semantic on
a more complicated model from a real application, and show similar behavior
to our simple voting example.

%We also ran Gibbs sampling on a subgraph of a real-world factor graph.
%This graph arises from the DeepDive data management
%system~\cite{wu2015incremental} solving a knowledge base population
%(KBP) problem.  Our results illustrate the effect of changing the semantic
%function, and show that more complicated models can also behave like the
%voting example.  They also indicate that, for a well-chosen model, Gibbs
%sampling can be a good choice for practical inference tasks.

% Figure \ref{figCoupling} shows the average coupling time of two random Gibbs
% sampling chains.  In order to produce these numbers, we initialized two
% chains independently at random, and then ran Gibbs sampling on them using the
% same source of randomness until the chains \emph{coupled}, that is, had the 
% same value.  We measured the average amount of time until this occurred.  This
% \emph{coupling time} is a measure of convergence that is related to the mixing
% time~\cite[p. 54]{levin2009markov}.  This experiment suggests that, for both
% real-world datasets, ratio semantics cause faster convergence than linear
% semantics.

\section{Conclusion}

This paper showed that for a class of factor graph templates,
hierarchical templates, Gibbs sampling mixes in polynomial time.  It
also introduced the graph property hierarchy width, and showed that
for graphs of bounded factor weight and hierarchy width, Gibbs
sampling converges rapidly.  These results may aid in better
understanding the behavior of Gibbs sampling for both template and
general factor graphs.

\subsection*{Acknowledgments} 
{
\footnotesize

Thanks to Stefano Ermon and Percy Liang for helpful conversations.

The authors acknowledge the support of:
DARPA FA8750-12-2-0335;
NSF IIS-1247701;
NSF CCF-1111943;
DOE 108845;
NSF CCF-1337375;
DARPA FA8750-13-2-0039; 
NSF IIS-1353606; ONR N000141210041
and N000141310129; NIH U54EB020405; Oracle; NVIDIA; Huawei; SAP Labs;
Sloan Research
Fellowship; Moore Foundation; American Family
Insurance; Google; and Toshiba.

% ``The views and conclusions contained herein are those of the authors and should not be interpreted as necessarily representing the official policies or endorsements, either expressed or implied, of DARPA, AFRL, NSF, ONR, NIH, or the U.S. Government.''
}

\begingroup
\renewcommand{\section}[2]{\subsubsection*{#2}}
\small
\bibliographystyle{plainnat} 
\bibliography{references}
\endgroup

\iftoggle{withappendix}{

\clearpage

\appendix

\section{Proof of Voting Program Rates}
Here, we prove the convergence rates stated in Theorem
\ref{thmVotingBounds}.  The strategy for the upper bound proofs
involves constructing a \emph{coupling} between the Gibbs sampler
and another process that attains the equilibrium distribution at each
step.  First, we restate the definition of mixing time in terms of the
\emph{total variation distance}, a quantity which we will use in the proofs
in this section.

\begin{definition}[Total Variation Distance]
The \emph{total variation distance}~\cite[p. 48]{levin2009markov} is a distance
metric between two probability measures $\mu$ and $\nu$ over probability space
$\Omega$ defined as
\[
  \tvdist{\mu - \nu} = \max_{A \subset \Omega} \Abs{\mu(A) - \nu(A)},
\]
that is, the maximum difference between the probabilities that $\mu$ and
$\nu$ assign to a single event.
\end{definition}

\begin{definition}[Mixing Time]
The \emph{mixing time} of a Markov chain is the first time $t$ at
which the estimated distribution $\mu_t$ is within total variation distance
$\frac{1}{4}$
of the true distribution~\cite[p. 55]{levin2009markov}.  That is,
\[
  t_{\mathrm{mix}}
  =
  \min \left\{
    t
    : 
    \tvdist{\mu_t - \pi}
    \le
    \frac{1}{4}
  \right\}.
\] 
\end{definition}

Next, we define a coupling~\citesec{coupling}.
\begin{definition}[Coupling]
A coupling of two random variables $X$ and $X'$ defined on some separate
probability spaces $\mathbb{P}$ and $\mathbb{P}'$ is any new probability
space $\hat{\mathbb{P}}$ over which there are two random variables
$\hat X$ and $\hat X'$ such that $X$ has the same distribution as $\hat X$
and $X'$ has the same distribution as $\hat X'$.
\end{definition}

Given a coupling of two Markov processes $X_k$ and
$X_k'$ with the same transition matrix, the \emph{coupling time} is defined as
the first time $T$
when $\hat X_k = \hat X_k'$.  The following theorem lets us bound the
total variance distance in terms of the coupling time.
\begin{theorem}[Theorem 5.2 in~\citet{levin2009markov}]
\label{thmCouplingTime}
For any coupling $(\hat X_k, \hat X_k')$ with coupling time $T$, if
we let $\nu_k$ and $\nu_k'$ denote the distributions of $X_k$ and $X_k'$
respectively, then
\[
  \norm{\nu_k - \nu_k'}_{\mathrm{tv}}
  \le
  \Prob{T > k}.
\]
\end{theorem}
All of the coupling examples in this section use a \emph{correlated flip
coupler}, which consists of two Gibbs samplers $\hat X$ and $\hat X'$, each
of which is running with the same random inputs.  Specifically, both
samplers choose to sample the same variable at each timestep.  Then,
if we define $p$ as the probability of sampling $1$ for $\hat X$, and
$p'$ similarly, it assigns both variables to $1$ with probability
$\min(p, p')$, assigns both variables to $0$ with probability
$\min(1-p, 1-p')$,
and assigns different values with probability $\Abs{p - p'}$.  If we 
initialize $\hat X$ with an arbitrary distribution $\nu$ and $\hat X'$ with
the stationary distribution $\pi$, it is
trivial to check that $\hat X_k$ has distribution $P^k \nu$, and
$\hat{X}_k'$ always has distribution $\pi$.  Applying Theorem
\ref{thmCouplingTime} results in
\[
  \norm{P^k \nu - \pi}_{\mathrm{tv}}
  \le
  \Prob{T > k}.
\]
Now, we prove the bounds stated in Theorem
\ref{thmVotingBounds} as a collection of four statements about the upper and
lower bounds of the mixing time.

\begin{statement}[UB for Voting: Logical and Ratio]
  \label{stmtVotingUBLog}
  For the voting example, assume that all the weights on the variables are
  bounded by $\Abs{w} \le M$, $\Abs{w_{T_x}} \le M$ and $\Abs{w_{F_x}} \le M$.
  Then for either the logical or ratio semantics,
  for sufficiently small constant $\epsilon$, 
  \[
    t_{\mathrm{mix}}(\epsilon)
    =
    O\left(n \log(n) \right).
  \]
\end{statement}
\begin{proof}
  Recall that, for the voting program, the weight of a world is
  \[
    W(q, t, f)
    =
    w q g(\norm{t}_1) - w q g(\norm{f}_1)
    +
    \sum_{x=1}^n w_{T(x)} t_x + \sum_{x=1}^n w_{F(x)} f_x,
  \]
  where $q \in \{-1, 1\}$ and $\{t_x, f_x \} \subseteq \{0, 1\}$.

  Consider a correlated-flip coupler running on this Markov process, producing
  two chains $X$ and $\bar X$.
  For either chain, if we sample a variable $T(x)$ at time $k$, and let $u$
  denote the number of
  $T(y)$ variables for $y \ne x$ that are true, then
  \[
    \Prob{T(x)^{(k+1)} = 1}
    =
    \left(1 + \exp\left(
      w q \left( g(u) - g(u+1) \right) - w_{T(x)}
    \right) \right)^{-1}.
  \]
  Since $g(x + 1) - g(x) \le 1$ for any $x$, it follows that
  \[
    \Prob{T(x)^{(k+1)} = 1}
    \ge
    \left(1 + \exp(2M) \right)^{-1}.
  \]
  We now define the constant
  \[
    p_M = \left(1 + \exp(2M) \right)^{-1}.
  \]
  Also, if $u$ non-$T(x)$ variables are true in the $X$ process, and $\bar u$
  variables are true in the $\bar X$ process, then the probability of the 
  processes $X$ and $\bar X$ not producing the same value for $Y_i$ is
  \begin{dmath*}
    p_D
    =
    \Abs{
      \Prob{T(x)^{(k+1)} = 1}
      -
      \Prob{\bar T(x)^{(k+1)} = 1}
    }
    =
    \Abs{
      \left(1 + \exp\left(
        w q \left( g(u) - g(u+1) \right) - w_{T(x)}
      \right) \right)^{-1}
      -
      \left(1 + \exp\left(
        w \bar q \left( g(\bar u) - g(\bar u+1) \right) - w_{T(x)}
      \right) \right)^{-1}
    },
  \end{dmath*}
  where the last statement follows from continuous differentiability of the
  function $h(x) = \left(1 + \exp(-x) \right)^{-1}$.
  Notice that, if $u = \Omega(n)$, then
  \[
    p_D
    =
    O(n^{-1})
  \]
  for both logical and ratio semantics.

  The same logic will show that, if we sample any $F(x)$, then
  \[
    \Prob{F(x)^{(k+1)} = 1}
    \ge
    p_M,
  \]
  and
  \[
    p_D = O(g(v + 1) - g(v) + g(\bar v + 1) - g(\bar v)),
  \]
  where $v$ and $\bar v$ are the number of non-$F(x)$ variables that are true
  in the $X$ and $\bar X$ processes respectively.  Again as above, if 
  $v = \Omega(n)$, then
  \[
    p_D
    =
    O(n^{-1}).
  \]

  Next, consider the situation if we sample $Q$.  In this case,
  \[
    \Prob{Q^{(k+1)} = 1}
    =
    \left(1 + \exp\left(
      2 w \left( g(\norm{t}_1) - g(\norm{f}_1) \right)
    \right) \right)^{-1}.
  \]
  Therefore, of we let $u$ denote the number of true $T(x)$
  variables and $v$ denote the number of true $F(x)$ variables, and similarly
  for $\bar u$ and $\bar v$, then the probability of the 
  processes $X$ and $\bar X$ not producing the same value for $Q$ is
  \begin{dmath*}
    p_D
    =
    \Abs{
      \Prob{Q^{(k+1)} = 1}
      -
      \Prob{\bar Q^{(k+1)} = 1}
    }
    =
    \Abs{
      \left(1 + \exp\left(
        2 w \left( f(v) - f(u) \right)
      \right) \right)^{-1}.
      -
      \left(1 + \exp\left(
        2 w \left( f(\bar v) - f(\bar u) \right)
      \right) \right)^{-1}.
    }
    =
    O\left(\Abs{
      f(v) - f(u) - f(\bar v) + f(\bar u),
    }\right)
  \end{dmath*}
  where as before, the last statement follows from continuous differentiability
  of the function $g(x) = \left(1 + \exp(-x) \right)^{-1}$.
  Furthermore, if all of $u$, $v$, $\bar u$, and $\bar v$ are $\Omega(n)$,
  then
  \[
    p_D = O(1).
  \]

  Now, assume that our correlated-flip coupler runs for $8 n \log n$ steps on
  the Gibbs sampler.  Let $E_1$ be the event that, after the first $4 n \log n$
  steps, each of the variables has been sampled at least once.  This will have
  probability at least $\frac{1}{2}$ by the coupon collector's problem.

  Next, let $E_2$ be the event that, after the first $4 n \log n$ steps,
  $\min(u, v, \bar u, \bar v) \ge \frac{p_M n}{2}$ for the next $4 n \log n$
  steps for both samplers.  After all
  the entries have been sampled, $u$, $v$, $\bar u$, and $\bar v$ will each be
  bounded from below by a binomial random variable with parameter
  $p_M$
  at each timestep, so from Hoeffding's inequality, the probability that
  this constraint will
  be violated by a sampler at any particular step is less than
  $\exp\left(-\frac{1}{2} p_M n \right)$.  Therefore,
  \[
    \Probc{E_2}{E_1}
    \ge
    \left(1 - \exp\left(-\frac{1}{2} p_M n \right) \right)^{4 n \log n}
    =
    \Omega(1).
  \]

  Let $E_3$ be the event that all the variables are resampled at least
  once between time $2 n \log n$ and $4 n \log n$.  Again this event has 
  probability at least $\frac{1}{2}$.

  Finally, let $C$ be the event that
  coupling occurs at time $4 n \log n$.  Given $E_1$, $E_2$, and $E_3$,
  this probability is equal to the probability that each variable
  coupled individually the last time it was sampled.  For all the $Y_i$,
  our analysis above showed that this probability is
  \[
    1 - p_D = 1 - O(n^{-1}),
  \]
  and for $Q$, this probability is
  \[
    1 - p_D = \Omega(1).
  \]
  Therefore,
  \[
    \Probc{C}{E_1, E_2, E_3}
    \ge
    \Omega(1) \left(1 - O(n^{-1}) \right)^{2n}
    =
    \Omega(1),
  \]
  and since
  \[
    \Prob{C}
    =
    \Probc{C}{E_1, E_2, E_3}
    \Prob{E_3}
    \Probc{E_2}{E_1}
    \Prob{E_1},
  \]
  and all these quantities are $\Omega(1)$,
  we can conclude that $\hat{\mathbb{P}}(C) = \Omega(1)$.

  Therefore, this process couples with at least some constant probability
  $P$ independent of $n$
  after $8 n \log n$ steps.  Since we can run this coupling argument
  independently an arbitrary number of times, it follows that, after
  $8 L n \log n$ steps, the probability of coupling will be at least
  $1 - (1 - P)^L$.  Therefore, by Theorem \ref{thmCouplingTime},
  for any initial distribution $\nu$,
  \[
    \| P^{8 L n \log n} \nu - \pi \|_{\mathrm{tv}}
    \le
    \hat{\mathbb{P}}(T > 8 L n \log n) \le (1 - P)^L.
  \]
  For any $\epsilon$, this will be less than $\epsilon$ when
  \[
    L \ge \frac{\log(\epsilon)}{\log(1 - p_C)},
  \]
  which occurs when
  $t \ge 8 n \log n \frac{\log(\epsilon)}{\log(1 - p_C)}$.
  Letting $\epsilon = \frac{1}{4}$,
  \[
    t_{\mathrm{mix}} = 8 n \log n \frac{\log(4)}{-\log(1 - p_C)}
  \]
  produces the desired result.
\end{proof}

\begin{statement}[LB for Voting: Logical and Ratio]
  For the voting example using either the logical or ratio
  semantics, a lower bound for the mixing time for sufficiently small
  constant values of $\epsilon$ is $\Omega(n \log n)$.
\end{statement}
\begin{proof}
  At a minimum, in order to converge, we must sample all the variables.  From
  the coupon collector's problem, this requires $\Omega(n \log n)$ time.
\end{proof}

\begin{statement}[UB for Voting: Linear]
  For the voting example, assume that all the weights on the variables are
  bounded by $\Abs{w} \le M$, $\Abs{w_{T(x)}} \le M$ and $\Abs{w_{F(x)}} \le M$.
  Then for linear semantics, 
  \[
    t_{\mathrm{mix}}
    =
    2^{O(n)}.
  \]
\end{statement}
\begin{proof}
  From our algebra above in the proof of Statement \ref{stmtVotingUBLog}, we
  know that at any timestep, the probability of not coupling the sampled
  variable is bounded; that is
  \[
    p_D = O(1).
  \]
  Therefore, if we run a correlated-flip coupler for $2 n \log n$ timesteps,
  the probability that coupling will have occurred is greater than the
  probability
  that all variables have been sampled (which is $\Omega(1)$) times
  the probability that all variables coupled the last time they were sampled.
  Thus if $C$ is the event that coupling occurs,
  \[
    \Prob{C}
    \ge
    \Omega(1) \left(1 - O(1)\right)^{2 n + 1}
    =
    \exp(-O(n)).
  \]
  Therefore, if we run for some $t = 2^{O(n)}$ timesteps, coupling
  will have occurred at least once with high probability.  It follows that 
  the mixing time is
  \[
    t_{\mathrm{mix}} = 2^{O(n)},
  \]
  as desired.
\end{proof}

\begin{statement}[LB for Voting: Linear]
  For the voting example using linear semantics and bounded weights, a
  lower bound for the mixing time of a worst-case model
  is $2^{\Omega(n)}$.
\end{statement}
\begin{proof}
  Consider a sampler with unary weights $w_{T(x)} = w_{F(x)} = 0$.
  Assume that it starts in a state where $Q = 1$, all the $T(x) = 1$,
  and all the $F(x) = 0$. 
  We will show that it takes an exponential amount of time until $Q = -1$.
  From above, the probability of flipping $Q$ will be
  \[
    \Prob{Q^{(k+1)} = -1}
    =
    \left(1 + \exp\left(
      2 w \left( \norm{t}_1) - \norm{f}_1) \right)
    \right) \right)^{-1}.
  \]
  Meanwhile, the probability to flip any $T(x)$ while $Q = 1$ is
  \[
    \Prob{T(x)^{(k+1)} = 0}
    =
    \left(1 + \exp(w) \right)^{-1}
    =
    p,
  \]
  for constant $p$, and the probability of flipping any $F(x)$ is similarly
  \[
    \Prob{F(x)^{(k+1)} = 1}
    =
    p.
  \]
  Now, consider the following events which could happen at any timestep.
  While $Q = 1$,
  let $E_T$ be the event that $\norm{t}_1 \le (1 - 2p) n$, and let $E_F$ be
  the event that $\norm{f}_1 \ge 2 p n$.  Since $\norm{t}_1$ and
  $\norm{f}_1$ are both bounded
  by binomial random variables with parameters $1 - p$ and $p$ respectively,
  Hoeffding's inequality states that, at any timestep,
  \[
    \Prob{E_T} = \Prob{\norm{t}_1 \le (1 - 2p) n} \le \exp(-2 p^2 n),
  \]
  and similarly
  \[
    \Prob{E_F} = \Prob{\norm{f}_1 \ge 2 p n} \le \exp(-2 p^2 n).
  \]
  Now, while these bounds are satisfied, let $E_Q$ be the event that
  $Q = -1$.  This will be bounded by
  \[
    \Prob{E_Q}
    =
    \left(
      1
      +
      \exp\left(2w (1 - 4p) n \right)
    \right)^{-1}
    \le
    \exp(-2w (1 - 4p) n).
  \]
  It follows that at any timestep,
  \[
    \Prob{E_T \vee E_F \vee E_Q} = \exp(-O(n)),
  \]
  so, at any timestep $k$,
  \[
    \Prob{Q^{(k)} = -1} = k \exp(-O(n)).
  \]
  However, by symmetry,
  under the stationary distribution $\pi$, this probability must be
  $\frac{1}{2}$.  Therefore, the total
  variation distance is bounded by
  \[
    \| P_t \nu - \pi \|_{\mathrm{tv}}
    \ge
    \frac{1}{2} - t \exp(-O(n)).
  \]
  So, for convergence to less than $\epsilon = \frac{1}{4}$, for example,
  we must require at least $2^{O(a)}$ steps.  This proves the statement.
\end{proof}

\section{Proof of Theorem \ref{thmPolynomialMixingTime}}

In this section, we prove the main result of the paper, Theorem
\ref{thmPolynomialMixingTime}.  First, we state some basic lemmas we will
need for the proof in Section \ref{secStatementLemmas}.
Then, we prove the main result inductively
in Section \ref{ssMainProofThmPMT}.  Finally, we prove
the lemmas in Section \ref{secProofLemmas}.

\subsection{Statement of Lemmas}
\label{secStatementLemmas}

Note that some of these lemmas are restated from the body of the paper.

\begin{definition}[Absolute Spectral Gap]
Let $P$ be the transition matrix of a Markov process.  Since it is a
Markov process, one of its eigenvalues, the dominant eigenvalue, must be $1$.
The \emph{absolute spectral gap} of the Markov process is the value
\[
  \gamma = 1 - \max_\lambda \Abs{\lambda},
\]
where the maximum is taken over all non-dominant eigenvalues of $P$.
\end{definition}

\lemmaReduceIndependent*

\lemmaRemoveOneFactor*

\lemmaNoFactors*

\begin{restatable}{lemma}{lemmaRelaxationTime}
\label{lemmaRelaxationTime}
Let $P$ be the transition matrix of a reversible, irreducible
Markov chain with state space $\Omega$ and absolute spectral gap $\gamma$,
and let
$\pi_{\mathrm{min}} = \min_{x \in \pi} \pi(x)$.  Then
\[
  t_{\mathrm{mix}}(\epsilon)
  \le
  -\log\left(\epsilon \pi_{\mathrm{min}} \right) \frac{1}{\gamma}.
\]
\end{restatable}
\begin{proof}
This is Theorem 12.3 from Markov Chains and Mixing Times
\cite[p. 155]{levin2009markov}, and a complete proof can be found in that book.
\end{proof}

% \begin{lemma}
% \label{lemmaHWChange}
% Let $G$ and $\bar G$ be two factor graphs, which differ only inasmuch as
% $G$ contains a single additional factor $\phi^*$.  Let $\mathcal{H}$ be the
% hierarchy that minimizes the weighted hierarchy width for $G$.
% Then if $\phi^*$ depends on a variable in the root node of $\mathcal{H}$,
% then
% \[
%   \mathsf{hw}(\bar G)
%   \le
%   \mathsf{hw}(G)
%   -
%   \Abs{w_{\phi^*}}.
% \]
% \end{lemma}

% \begin{lemma}
% \label{lemmaSubgraph}
% Let $G$ be a factor graph, and let $\bar G$ be some subgraph of $G$; that is,
% the variables of $\bar G$ are a subset of the variables of $G$, and the
% factors of $\bar G$ are also a subset of the factors of $G$.  Then,
% \[
% \mathsf{hw}(\bar G)
% \le
% \mathsf{hw}(G);
% \]
% \end{lemma}

% \begin{definition}
% Let $G$ be a factor graph with variables $V$, factors $\Phi$, and total weight
% for $I \subseteq V$,
% \[
%   W(I) = \sum_{\phi \in \Phi} \phi(I \cap A_\phi),
% \]
% where $A_{\phi}$ is the set of variables that $\phi$ depends on.
% Then, define the \emph{induced hypergraph} $\mathcal{H}_G$ of $G$ to be the
% weighted hypergraph
% with nodes $V$, hyperedges $\{ A_{\phi} \}$, and weights
% \[
%   \omega(A_{\phi})
%   =
%   \max_{I \subset A_{\phi}} \phi(I)
%   -
%   \min_{I \subset A_{\phi}} \phi(I).
% \]
% That is, each of the
% factors of $G$ produces a hyperedge of $\mathcal{H}_G$.  We also let
% $M_G$ denote the maximum weight of any hyperedge of $\mathcal{H}_G$.
% \end{definition}

% \begin{lemma}
% For any factor graph $G$,
% \[
%   \mathsf{whw}(\mathcal{H}_G)
%   \le
%   M_G \mathsf{hw}(\mathcal{H}_G).
% \]
% \end{lemma}

\subsection{Main Proof}
\label{ssMainProofThmPMT}

We achieve the proof of Theorem \ref{thmPolynomialMixingTime}
in two steps.  First, we inductively bound the absolute
spectral gap of Gibbs sampling on factor graphs in terms of the hierarchy
width.  Then, we use the bound on the spectral gap to bound the mixing time.

\begin{lemma}
\label{lemmaBoundASG}
If $G$ is a factor graph with $n$ variables, maximum factor weight $M$ and
hierarchy width $h$, then Gibbs sampling running on $G$ will have
absolute spectral gap $\gamma$, where
\[
  \gamma
  \ge
  \frac{1}{n}
  \exp\left( -3 h M \right).
\]
\end{lemma}
\begin{proof}
We will prove this result by multiple induction on the number of variables
and number of factors of $G$.  In what follows, we assume that the statement
holds for all graphs with either fewer variables and no more factors than $G$,
or fewer factors and no more variables than $G$.

Consider the connectedness of $G$.  There are three possibilities:
\begin{enumerate}
  \item $G$ has one variable and no factors.
  \item $G$ is connected and has at least one factor.
  \item $G$ is disconnected.
\end{enumerate}
We consider these cases separately.

\paragraph{Case 1}
If $G$ has no factors and one variable, then by Lemma \ref{lemmaNoFactors},
its spectral gap will be
\[
  \gamma = 1.
\]
Also, if $G$ has no factors, by (\ref{eqnHWBase}), its hierarchy width is
$h = 0$.  Therefore,
\begin{dmath*}
  \gamma
  =
  \frac{1}{1} \exp(0)
  =
  \frac{1}{n} \exp(-3 hM),
\end{dmath*}
which is the desired result.

\paragraph{Case 2}
Next, we consider the case where $G$ is connected and has at least one factor.
In this case, by (\ref{eqnHWConnected}), its hierarchy width is
\[
  \mathsf{hw}(G)
  =
  1 + \min_{e \in E} \mathsf{hw}(\langle N, E - \{e\}\rangle).
\]
Let $e$ be an edge that minimizes this quantity, and let $\bar G$ denote the
factor graph that results
from removing the corresponding factor from $G$.  Clearly,
\[
  \mathsf{hw}(G)
  =
  1 + \mathsf{hw}(\bar G).
\]
Since $G$ and $\bar G$ differ by only one factor,
by Lemma \ref{lemmaRemoveOneFactor}, 
\[
  \gamma \ge \bar \gamma \exp\left(-3M \right).
\]
Furthermore, since $\bar G$ has fewer factors than $G$, by the inductive
hypothesis,
\[
  \bar \gamma
  \ge
  \frac{1}{n}
  \exp\left( -3M \mathsf{hw}(\bar G) \right).
\]
Therefore,
\begin{dmath*}
  \gamma
  \ge
  \bar \gamma \exp\left(-M \right)
  \ge
  \frac{1}{n}
  \exp\left( -3M \mathsf{hw}(\bar G) \right)
  \exp\left(-3M \right)
  =
  \frac{1}{n}
  \exp\left( -3M \left( 1 + \mathsf{hw}(\bar G) \right) \right)
  =
  \frac{1}{n}
  \exp\left( -3M \mathsf{hw}(G) \right),
\end{dmath*}
which is the desired result.

\paragraph{Case 3}
Finally, consider the case where $G$ is disconnected.  Let
$G_1, G_2, \ldots, G_m$ be the connected components of $G$ for some $m \ge 2$.
In this case, by (\ref{eqnHWDisconnected}), its hierarchy width is
\[
  \mathsf{hw}(G) = \max_i \mathsf{hw}(G_i).
\]
By Lemma \ref{lemmaReduceIndependent},
\[
  \gamma
  =
  \min_{i \le m} \frac{n_i}{n} \gamma_i,
\]
where $n_i$ are the sizes of the $G_i$ and $\gamma_i$ are their absolute
spectral gaps.  We further know that each of the $G_i$ has fewer nodes than
$G$, so by the inductive hypothesis,
\[
  \gamma_i
  \ge
  \frac{1}{n_i}
  \exp\left( -3M \mathsf{hw}(G_i) \right).
\]
Therefore,
\begin{dmath*}
  \gamma
  =
  \min_{i \le m} \frac{n_i}{n} \gamma_i
  \ge
  \min_{i \le m} \frac{n_i}{n}
  \frac{1}{n_i}
  \exp\left( -3M \mathsf{hw}(G_i) \right)
  =
  \min_{i \le m} \frac{1}{n}
  \exp\left( -3M \mathsf{hw}(G_i) \right)
  =
  \frac{1}{n}
  \exp\left( -3M \max_{i \le m} \mathsf{hw}(G_i) \right)
  =
  \frac{1}{n}
  \exp\left( -3M \mathsf{hw}(G) \right),
\end{dmath*}
which is the desired result.

Since the statement holds in all three cases, by induction it will hold for
all factor graphs.  This completes the proof.
\end{proof}

Now, we restate and prove the main theorem.

\thmPolynomialMixingTime*
\begin{proof}
From Lemma \ref{lemmaRelaxationTime}, we have that
\[
  t_{\mathrm{mix}}(\epsilon)
  \le
  -\log\left(\epsilon \pi_{\mathrm{min}} \right) \frac{1}{\gamma}.
\]
For our factor graph $G$,
\begin{dmath*}
\pi_{\mathrm{min}}
=
\min_I \pi(I)
=
\frac{\min_I \exp(W(I))}{\sum_J \exp(W(J))}.
\end{dmath*}
Since there are only at most $s^n$ worlds, it follows that
\begin{dmath*}
\pi_{\mathrm{min}}
\ge
\frac{\min_I \exp(W(I))}{s^n \max_J \exp(W(J))}
=
s^{-n} \exp\left(
  \min_I W(I)
  -
  \max_J W(J)
\right).
\end{dmath*}
Expanding the world weight in terms of the factors,
\begin{dmath*}
\pi_{\mathrm{min}}
\ge
s^{-n} \exp\left(
  \min_I \sum_{\phi \in \Phi} \phi(I)
  -
  \max_J \sum_{\phi \in \Phi} \phi(J)
\right)
\ge
s^{-n} \exp\left(
  \sum_{\phi \in \Phi} \left(
    \min_I \phi(I)
    -
    \max_J \phi(J)
  \right)
\right).
\end{dmath*}
Applying our maximum factor weight bound,
\begin{dmath*}
\pi_{\mathrm{min}}
\ge
s^{-n} \exp\left(
  -\sum_{\phi \in \Phi} M
\right)
=
s^{-n} \exp\left(-eM\right).
\end{dmath*}
Substituting this into the expression from Lemma \ref{lemmaRelaxationTime}
produces
\begin{dmath*}
  t_{\mathrm{mix}}(\epsilon)
  \le
  -\log\left(
    \frac{\epsilon}{s^n} \exp\left(-eM\right)
  \right) \frac{1}{\gamma}
  =
  \left(
    n \log(s) + e M - \log(\epsilon)
  \right)
  \frac{1}{\gamma},
\end{dmath*}
and substituting the result from Lemma \ref{lemmaBoundASG} gives
\begin{dmath*}
  t_{\mathrm{mix}}(\epsilon)
  \le
  \left(
    n \log(s) + e M - \log(\epsilon)
  \right)
  n \exp(3 h M).
\end{dmath*}
Substituting $\epsilon = \frac{1}{4}$ gives the desired result.
\end{proof}

\subsection{Proofs of Lemmas}
\label{secProofLemmas}

In this statement, we will restate and prove the lemmas stated above in
Section \ref{secStatementLemmas}.

\lemmaReduceIndependent*
\begin{proof}
Lemma \ref{lemmaReduceIndependent} follows directly from Corollary 12.12 in
Markov Chains and Mixing Times \cite[p. 161]{levin2009markov}.  For
completeness, we restate the proof here.

Since $G$ is disconnected, Gibbs sampling on $G$ is equivalent to running
$m$ independent Gibbs samplers on the $G_i$, where at each timestep, the Gibbs
sampler of $G_i$ is updated if a variable in $G_i$ is chosen; this occurs with
probability $\frac{n_i}{n}$.  Therefore, if we let $P_i$ denote the Markov
transition matrix of Gibbs sampling on $G_i$, we can write the transition
matrix $P$ of Gibbs sampling on $G$ as
\[
  P(x, y)
  =
  \sum_{i=1}^m \frac{n_i}{n} P_i(x_i, y_i)
  \prod_{j \ne i} I(x_j, y_j),
\]
where $x$ and $y$ are worlds on $G$, 
where $x_i$ and $y_i$ are subsets of the variables of $x$ and $y$ respectively
that correspond to the graph $G_i$, and $I$ is the identity matrix
($I(x,y) = 1$ if $x = y$, and $I(x,y) = 0$ otherwise).

Now, we, for each $i$, we let $f_i$ be some eigenvector of $P_i$ with
eigenvalue $\lambda_i$, and define
\[
  f(x)
  =
  \prod_{i=1}^m f_i(x_i).
\]
It follows that
\begin{dmath*}
  (Pf)(y)
  =
  \sum_x
  f(x) P(x, y)
  =
  \sum_x
  f(x)
  \sum_{i=1}^m \frac{n_i}{n} P_i(x_i, y_i)
  \prod_{j \ne i} I(x_j, y_j)
  =
  \sum_x
  \sum_{i=1}^m \frac{n_i}{n} f_i(x_i) P_i(x_i, y_i)
  \prod_{j \ne i} f_j(x_j) I(x_j, y_j)
  =
  \sum_{i=1}^m \frac{n_i}{n} 
  \sum_{x_i} f_i(x_i) P_i(x_i, y_i)
  \prod_{j \ne i} 
  \sum_{x_j} f_j(x_j) I(x_j, y_j)
  =
  \sum_{i=1}^m \frac{n_i}{n} 
  \lambda_i f_i(y_i)
  \prod_{j \ne i} 
  f_j(y_j)
  =
  f(y)
  \sum_{i=1}^m \frac{n_i}{n} \lambda_i;
\end{dmath*}
if follows that $f$ is an eigenvector of $P$ with eigenvalue
\begin{equation}
  \label{eqnLambdaFormRI}
  \lambda = \sum_{i=1}^m \frac{n_i}{n} \lambda_i.
\end{equation}
Furthermore, it is clear that such eigenvectors will form a basis for the space
of possible distributions on the worlds of $G$.  Therefore, all eigenvalues of
$P$ will be of the form (\ref{eqnLambdaFormRI}), and so
\[
  \Abs{\lambda}
  \le
  \sum_{i=1}^m \frac{n_i}{n} \Abs{\lambda_i}.
\]
Therefore,
\begin{dmath*}
  \gamma
  =
  \min_{\lambda \ne 1}
  \left( 1 - \Abs{\lambda} \right)
  \ge
  \min_{\lambda \ne 1} 
  \sum_{i=1}^m \frac{n_i}{n} \left(1 - \Abs{\lambda_i} \right)
  \ge
  \min_i
  \min_{\lambda_i \ne 1}
  \frac{n_i}{n} \left(1 - \Abs{\lambda_i} \right)
  =
  \min_i
  \frac{n_i}{n} \gamma_i,
\end{dmath*}
which is the desired result.
\end{proof}

To prove Lemma \ref{lemmaRemoveOneFactor}, we will need to first prove two
other lemmas.

\begin{lemma}
\label{lemmaROFDistBound}
Let $G$ and $\bar G$ be two factor graphs, which differ only inasmuch as
$G$ contains a single additional factor $\phi^*$ with maximum factor weight $M$.
Then, if $\pi$ and $\bar \pi$ denote the distributions on worlds for these
graphs,
\[
  \exp(-M) \bar \pi(x) \le \pi(x) \le \exp(M) \bar \pi(x). 
\]
\end{lemma}
\begin{proof}
From the definition of the distribution function,
\begin{dmath*}
  \pi(x)
  =
  \frac{
    \exp(W(x))
  }{
    \sum_{z \in \Omega} \exp(W(z))
  }
  =
  \frac{
    \exp(\bar W(x) + \phi^*(x))
  }{
    \sum_{z \in \Omega} \exp(\bar W(z) + \phi^*(z))
  }
  \ge
  \frac{
    \exp(\bar W(x) + \min_I \phi^*(I))
  }{
    \sum_{z \in \Omega} \exp(\bar W(z) + \max_I \phi^*(I))
  }
  =
  \frac{
    \exp(\bar W(x))
  }{
    \sum_{z \in \Omega} \exp(\bar W(z))
  }
  \exp\left(\min_I \phi^*(I) - \max_I \phi^*(I) \right)
  =
  \exp(-M) \bar \pi(x).
\end{dmath*}
In the other direction,
\begin{dmath*}
  \pi(x)
  =
  \frac{
    \exp(\bar W(x) + \phi^*(x))
  }{
    \sum_{z \in \Omega} \exp(\bar W(z) + \phi^*(z))
  }
  \le
  \frac{
    \exp(\bar W(x) + \max_I \phi^*(I))
  }{
    \sum_{z \in \Omega} \exp(\bar W(z) + \min_I \phi^*(I))
  }
  =
  \frac{
    \exp(\bar W(x))
  }{
    \sum_{z \in \Omega} \exp(\bar W(z))
  }
  \exp\left(\max_I \phi^*(I) - \min_I \phi^*(I) \right)
  =
  \exp(M) \bar \pi(x).
\end{dmath*}
This proves the lemma.
\end{proof}

\begin{lemma}
\label{lemmaROFTransBound}
Let $G$ and $\bar G$ be two factor graphs, which differ only inasmuch as
$G$ contains a single additional factor $\phi^*$ with maximum factor weight $M$.
Then, if $P$ and $\bar P$ denote the transition matrices for Gibbs sampling on
these graphs, for all $x \ne y$,
\[
  \exp(-M) \bar P(x, y) \le P(x, y) \le \exp(M) \bar P(x, y). 
\]
\end{lemma}
\begin{proof}
For Gibbs sampling on $G$, we know that, if worlds $x$ and $y$ differ in
the value of only one variable, and $L$ is the set of worlds that only differ
from $x$ and $y$ in the same variable, then
\[
  P(x, y)
  =
  \frac{1}{n}
  \frac{
    \exp(W(y))
  }{
    \sum_{l \in L} \exp(W(l))
  }
\]
If $\bar P$ is the transition matrix for Gibbs sampling on $\bar G$, then
the same argument will show that
\[
  \bar P(x, y)
  =
  \frac{1}{n}
  \frac{
    \exp(\bar W(y))
  }{
    \sum_{l \in L} \exp(\bar W(l))
  }.
\]
Therefore,
\begin{dmath*}
  P(x, y)
  =
  \frac{1}{n}
  \frac{
    \exp(\bar W(y) + \phi^*(y))
  }{
    \sum_{l \in L} \exp(\bar W(l) + \phi^*(y))
  }
  \ge
  \frac{1}{n}
  \frac{
    \exp(\bar W(y) + \min_I \phi^*(I))
  }{
    \sum_{l \in L} \exp(\bar W(l) + \max_I \phi^*(I))
  }
  =
  \frac{1}{n}
  \frac{
    \exp(\bar W(y))
  }{
    \sum_{l \in L} \exp(\bar W(l))
  }
  \exp\left(\min_I \phi^*(I) - \max_I \phi^*(I) \right)
  =
  \exp(-M) \bar P(x, y).
\end{dmath*}
Similarly,
\begin{dmath*}
  P(x, y)
  =
  \frac{1}{n}
  \frac{
    \exp(\bar W(y) + \phi^*(y))
  }{
    \sum_{l \in L} \exp(\bar W(l) + \phi^*(y))
  }
  \le
  \frac{1}{n}
  \frac{
    \exp(\bar W(y) + \max_I \phi^*(I))
  }{
    \sum_{l \in L} \exp(\bar W(l) + \min_I \phi^*(I))
  }
  =
  \frac{1}{n}
  \frac{
    \exp(\bar W(y))
  }{
    \sum_{l \in L} \exp(\bar W(l))
  }
  \exp\left(\max_I \phi^*(I) - \min_I \phi^*(I) \right)
  =
  \exp(M) \bar P(x, y).
\end{dmath*}
On the other hand, if $x$ and $y$ differ in more than one variable, then
\[
  P(x, y) = \bar P(x, y) = 0.
\]
Therefore, the lemma holds for any $x \ne y$.
\end{proof}

Now, we restate and prove Lemma \ref{lemmaRemoveOneFactor}.

\lemmaRemoveOneFactor*
\begin{proof}
The proof is similar to proofs using the Dirichlet form to bound the spectral
gap~\cite[p. 181]{levin2009markov}.

Recall that Gibbs sampling is a reversible Markov chain.  That is, if $P$ is
the transition matrix associated with Gibbs sampling on $G$, and $\pi$ is its
stationary distribution,
\[
  \pi(x) P(x, y) = \pi(y) P(y, x).
\]
If we let $\Pi$ be the diagonal matrix such that $\Pi(x, x) = \pi(x)$, then
we can write this as
\[
  P \Pi = \Pi P^T;
\]
that is, $A = P \Pi$ is a symmetric matrix.  Now, consider the 
form
\[
  \alpha(f)
  =
  \frac{1}{2}
  \sum_{x, y}
  \left(
    f(x)
    -
    f(y)
  \right)^2
  \pi(x)
  P(x, y).
\]
Expanding this produces
\begin{dmath*}
  \alpha(f)
  =
  \frac{1}{2}
  \sum_{x, y}
  f(x)^2
  \pi(x)
  P(x, y)
  -
  \sum_{x, y}
  f(x) f(y)
  \pi(x)
  P(x, y)
  +
  \frac{1}{2}
  \sum_{x, y}
  f(y)^2
  \pi(x)
  P(x, y).
\end{dmath*}
By reversibility of the chain,
\begin{dmath*}
  \alpha(f)
  =
  \frac{1}{2}
  \sum_{x, y}
  f(x)^2
  \pi(x)
  P(x, y)
  -
  \sum_{x, y}
  f(x) f(y)
  \pi(x)
  P(x, y)
  +
  \frac{1}{2}
  \sum_{x, y}
  f(y)^2
  \pi(y)
  P(y, x)
  =
  \sum_{x, y}
  f(x)^2
  \pi(x)
  P(x, y)
  -
  \sum_{x, y}
  f(x) f(y)
  \pi(x)
  P(x, y)
  =
  \sum_x
  f(x)^2
  \pi(x)
  \sum_y
  P(x, y)
  -
  \sum_{x, y}
  f(x) f(y)
  \pi(x)
  P(x, y)
  =
  \sum_x
  \pi(x)
  f(x)^2
  -
  \sum_{x, y}
  f(x) f(y) A(x, y)
  =
  f^T \Pi f - f^T A f
  =
  f^T (I - P) \Pi f.
\end{dmath*}
Similarly, if we define
\[
  \beta(f)
  =
  \frac{1}{2}
  \sum_{x, y}
  \left(
    f(x)
    +
    f(y)
  \right)^2
  \pi(x)
  P(x, y),
\]
then the same logic will show that
\[
  \beta(f)
  =
  f^T (I + P) \Pi f.
\]
We define $\bar A$, $\bar \alpha$ and $\bar \beta$ similarly for Gibbs
sampling on $\bar G$.

Now, by
Lemmas \ref{lemmaROFDistBound} and \ref{lemmaROFTransBound},
\begin{dmath*}
  \alpha(f)
  =
  \frac{1}{2}
  \sum_{x, y}
  \left(
    f(x)
    -
    f(y)
  \right)^2
  \pi(x)
  P(x, y)
  \ge
  \frac{1}{2}
  \sum_{x, y}
  \left(
    f(x)
    -
    f(y)
  \right)^2
  \exp(-2M)
  \bar \pi(x)
  \bar P(x, y)
  =
  \exp(-2M) \bar \alpha(f).
\end{dmath*}
It follows that
\[
  f^T (I - P) \Pi f \ge \exp(-2M) f^T (I - \bar P) \bar \Pi f.
\]
We also notice that, by Lemma \ref{lemmaROFDistBound},
\begin{dmath*}
  f^T \left(\Pi - \pi \pi^T \right) f
  =
  f^T \Pi f - (\pi^T f)^2
  =
  \sum_x f(x)^2 \pi(x)
  -
  \left( \sum_y f(y) \pi(y) \right)^2
  =
  \sum_x \left(
    f(x) - \sum_y f(y) \pi(y)
  \right)^2 \pi(x)
  \le
  \sum_x \left(
    f(x) - \sum_y f(y) \bar \pi(y)
  \right)^2 \pi(x)
  \le
  \exp(M)
  \sum_x \left(
    f(x) - \sum_y f(y) \bar \pi(y)
  \right)^2 \bar \pi(x)
  =
  \exp(M)
  \left(
  \sum_x f(x)^2 \bar \pi(x)
  -
  \left( \sum_y f(y) \bar \pi(y) \right)^2
  \right)
  =
  \exp(M) f^T \left(\bar \Pi - \bar \pi \bar \pi^T \right) f
\end{dmath*}
Therefore,
\begin{equation}
  \label{eqnROFMin2}
  \frac{
    f^T (I - P) \Pi f
  }{
    f^T \left( \Pi - \pi \pi^T \right) f
  }
  \ge
  \exp(-3M)
  \frac{
    f^T (I - \bar P) \bar \Pi f
  }{
    f^T \left( \bar \Pi - \bar \pi \bar \pi^T \right) f
  }.
\end{equation}

Now, let's explore the quantity
\begin{equation}
  \label{eqnROFMin1}
  \min_f
  \frac{
    f^T (I - P) \Pi f
  }{
    f^T \left( \Pi - \pi \pi^T \right) f
  }.
\end{equation}
First, we notice that
\[
  (I - P) \Pi \mathbf{1}
  =
  (I - P) \pi
  =
  \pi - \pi
  =
  0,
\]
so $\mathbf{1}$ is an eigenvector of $(I - P) \Pi$ with eigenvalue $0$.
Furthermore,
\[
  \left( \Pi - \pi \pi^T \right) \mathbf{1}
  =
  \left( \pi - \pi \pi^T \mathbf{1} \right)
  =
  \pi - \pi
  =
  0,
\]
so $\mathbf{1}$ is also an eigenvector of $(I - P) \Pi$ with eigenvalue $0$.
It follows that (\ref{eqnROFMin1}) is invariant to additions of the vector
$\mathbf{1}$ to $f$, and in particular we can therefore choose to minimize over
only those vectors for which $\pi^T f = 0$.  Therefore,
\begin{dmath*}
  \min_f
  \frac{
    f^T (I - P) \Pi f
  }{
    f^T \left( \Pi - \pi \pi^T \right) f
  }
  =
  \min_{\pi^T g = 0}
  \frac{
    g^T (I - P) \Pi g
  }{
    g^T \Pi g
  }
  =
  \min_{\pi^T g = 0}
  \frac{
    g^T \Pi^{\frac{1}{2}} \Pi^{-\frac{1}{2}}
    (I - P)
    \Pi^{\frac{1}{2}} \Pi^{\frac{1}{2}}
    g
  }{
    g^T \Pi g
  }.
\end{dmath*}
If we let $h = \Pi^{\frac{1}{2}} g$, then
\begin{dmath*}
  \min_f
  \frac{
    f^T (I - P) \Pi f
  }{
    f^T \left( \Pi - \pi \pi^T \right) f
  }
  =
  \min_{\mathbf{1}^T \Pi^{\frac{1}{2}} h = 0}
  \frac{
    h^T \Pi^{-\frac{1}{2}}
    (I - P)
    \Pi^{\frac{1}{2}} h
  }{
    h^T h
  }.
\end{dmath*}
Notice that $\Pi^{-\frac{1}{2}} (I - P) \Pi^{\frac{1}{2}}$ is similar to
$I - P$, and so it must have the same eigenvalues.  Since it is symmetric,
it has an orthogonal complement of eigenvectors, and the eigenvector that
corresponds to $\lambda = 0$ is
\[
  \left(\Pi^{-\frac{1}{2}} (I - P) \Pi^{\frac{1}{2}} \right)
  \left(
    \Pi^{\frac{1}{2}} \mathbf{1}
  \right)
  =
  \Pi^{-\frac{1}{2}} (I - P) \pi
  =
  \Pi^{-\frac{1}{2}} (\pi - \pi)
  =
  0.
\]
The expression above is minimizing over all vectors orthogonal to this
eigenvector; it follows that the minimum will be the second smallest eigenvalue
of $I - P$.  This eigenvalue will be $1 - \lambda_2$, where $\lambda_2$ is
the second-largest eigenvalue of $P$.  So,
\[
  \min_f
  \frac{
    f^T (I - P) \Pi f
  }{
    f^T \left( \Pi - \pi \pi^T \right) f
  }
  =
  1 - \lambda_2.
\]
The same argument will show that
\[
  \min_f
  \frac{
    f^T (I - \bar P) \bar \Pi f
  }{
    f^T \left( \bar \Pi - \bar \pi \bar \pi^T \right) f
  }
  =
  1 - \bar \lambda_2.
\]
Therefore, minimizing both sides in (\ref{eqnROFMin2}) produces
\[
  1 - \lambda_2
  \ge
  \exp(-3M)
  \left(1 - \bar \lambda_2 \right).
\]
And applying the definition of absolute spectral gap,
\[
  1 - \lambda_2
  \ge
  \exp(-3M)
  \bar \gamma
\]

Now, a similar argument to the above will show that
\[
  \frac{
    f^T (I + P) \Pi f
  }{
    f^T \Pi f
  }
  \ge
  \exp(-3M)
  \frac{
    f^T (I + \bar P) \bar \Pi f
  }{
    f^T \bar \Pi f
  }.
\]
And we will be able to conclude that
\[
  \min_f
  \frac{
    f^T (I + P) \Pi f
  }{
    f^T \Pi f
  }
  =
  1 - \lambda_n,
\]
where $\lambda_n$ is the (algebraically) smallest eigenvalue of $P$,
and similarly for $\bar \lambda_n$.  Therefore, we can conclude that
\[
  1 + \lambda_n
  \ge
  \exp(-3M)
  \left(1 + \bar \lambda_n \right),
\]
and from the definition of absolute spectral gap,
\[
  1 + \lambda_n
  \ge
  \exp(-3M)
  \bar \gamma.
\]
Therefore,
\[
  \gamma
  \ge
  \min(1 - \lambda_2, 1 + \lambda_n)
  \ge
  \exp(-3M) \bar \gamma.
\]
This proves the lemma.
\end{proof}

Next, we restate and prove Lemma \ref{lemmaNoFactors}

\lemmaNoFactors*
\begin{proof}
Gibbs sampling on a factor graph with one variable and no factors will
have transition matrix
\[
  P = \pi \mathbf{1}^T,
\]
where $\pi$ is the stationary distribution, and $\mathbf{1}$ is the vector
of all $1$s.  That is, the process achieves the stationary distribution in a
single step.  The only eigenvalues of this matrix are $0$ and $1$, so the
absolute spectral gap will be $1$, as desired.
\end{proof}

\section{Proofs of Other Results}

In this section, we will prove the other results about hierarchy decomposition
stated in Section \ref{secHierarchyWidth}.
First, we express it in
terms of a hypergraph decomposition; this is how hypertree width is
defined so it is useful for comparison.

\begin{restatable}[Hierarchy Decomposition]{definition}{defHierarchyDecomp}
\label{defHierarchyDecomp}
A hierarchy decomposition of a hypergraph $G = \langle N, E \rangle$ is a 
rooted tree
$T$ where each node $v$ of $T$ is labeled with a set $\chi(v)$ of edges of
$G$, that satisfies the following conditions:
(1) for each edge $e \in E$, there is some node $v$ of $T$ such that
  $e \in \chi(v)$;
(2) if for two nodes $u$ and $v$ of $T$ and for some edge $e \in E$, both
  $e \in \chi(u)$ and $e \in \chi(v)$, then for all nodes $w$ on the (unique)
  path between $u$ and $v$ in $T$, $e \in \chi(w)$;
(3) for every pair of edges $e \in E$ and $f \in E$, if
  $e \cap f \ne \emptyset$, then there exists a node $v$ of $T$ such that
  $\{e, f\} \subseteq \chi(v)$; and
(4) if node $u$ is the parent of node $v$ in $T$, then $u \subseteq v$.

The width of a hierarchy decomposition is the size of the largest node in $T$,
that is, $\max_v \Abs{\chi(v)}$.
\end{restatable}

\begin{restatable}{statement}{stmtHierarchyDecomp}
\label{stmtHierarchyDecomp}
The \emph{hierarchy width} of a hypergraph $G$ is
equal to the minimum width of a hierarchy decomposition of $G$.
\end{restatable}

In this section, we denote the \emph{width} of a hierarchy decomposition $T$ as
$\mathsf{w}(T)$.

\begin{lemma}
\label{lemmaHDSingleChild}
Let $T$ be a hierarchy decomposition of a hypergraph $G$.  Let $u$ be a node of
$T$ that has exactly one child $v$, and let $e$ be the edge connecting $u$ and
$v$.  Let $\bar T$ be the graph minor of $T$ that is formed by removing $u$
from $T$ and, if applicable, connecting $v$ with the parent of $u$; assume that
the nodes $\bar T$ have the same labeling as the corresponding nodes in $T$.
Then $\bar T$ is a hierarchy decomposition of $G$, and
\[
  \mathsf{w}(T) = \mathsf{w}(\bar T).
\]
\end{lemma}
\begin{proof}
We validate the conditions of the definition of hierarchy decomposition
individually.
\begin{enumerate}
  \item Since $T$ is a hierarchy decomposition of $G$, by condition (1),
    for any hyperedge $e$ of
    $G$, there exists a node $x$ of $T$ such that $e \in \chi_T(x)$.  If
    $x \ne u$, then $e \in \chi_{\bar T}(x)$.  Otherwise, by condition (4) of
    $T$, it will hold that
    $\chi_T(u) \subseteq \chi_T(v) = \chi_{\bar T}(v)$, so
    $e \in \chi_{\bar T}(v)$. Since this is true for any hyperedge $e$, the
    condition holds.
  \item Since we constructed $\bar T$ by removing a node from $T$ and connecting
    its neighbors, it follows that
    if the path from node $x$ to node $y$ in $\bar T$ passes through a
    node $z$, then the path from $x$ to $y$ in $T$ also passes through $z$.
    The condition then follows directly from condition (2) of $T$.
  \item The argument here is the same as the argument for condition (1).
  \item This condition follows directly from transitivity of the subset
    relation.
\end{enumerate}
Therefore, $\bar T$ is a hierarchy decomposition of $G$.  To show that 
$\mathsf{w}(T) = \mathsf{w}(\bar T)$, it suffices to notice that removing
node $u$ from $T$ can't change its width, since node $v$ has at least as many
hyperedges in its labeling as $u$.  This proves the lemma.
\end{proof}

\begin{lemma}
\label{lemmaHDSameLabel}
Let $T$ be a hierarchy decomposition of a hypergraph $G$.  Let $u$ be a node of
$T$, and let $v$ be a child of $T$ such that $\chi_T(u) = \chi_T(v)$.
Let $\bar T$ be the graph minor of $T$ that is formed by identifying $u$ and
$v$ by contracting along their connecting edge. Assume that
the nodes $\bar T$ have the same labeling as the corresponding nodes in $T$,
and that the new node, $w$, has $\chi_{\bar T}(w) = \chi_T(u) = \chi_T(v)$.
Then $\bar T$ is a hierarchy decomposition of $G$, and
\[
  \mathsf{w}(T) = \mathsf{w}(\bar T).
\]
\end{lemma}
\begin{proof}
We validate the conditions of the definition of hierarchy decomposition
individually.
\begin{enumerate}
  \item The set of labelings of $\bar T$ is exactly the same as the set of
    labelings of $T$.  Since $T$ is a hierarchy decomposition of $G$, the
    condition follows directly from condition (1) of $T$.
  \item Similarly, the set of labelings of any path in $\bar T$ is the same
    as the set of labeling of the corresponding path in $T$.  The condition
    then follows directly from condition (2) of $T$.
  \item The argument here is the same as the argument for condition (1).
  \item The labeling of every node's parent in $\bar T$ will be the same as
    the labeling of that node's parent in $T$.  Therefore, the condition follows
    from condition (3) of $T$.
\end{enumerate}
Therefore, $\bar T$ is a hierarchy decomposition of $G$.  To show that 
$\mathsf{w}(T) = \mathsf{w}(\bar T)$, it suffices to notice the set of
labelings of nodes of $\bar T$ is the same as the set of labelings of nodes of
$T$; therefore they must have the same width.
\end{proof}

\begin{lemma}
\label{lemmaHierarchyDecompForward}
For any factor graph $G$ and for any hierarchy decomposition $T$ of $G$,
\[
  \mathsf{w}(T) \ge \mathsf{hw}(G).
\]
\end{lemma}
\begin{proof}
As in the proof of Lemma \ref{lemmaBoundASG}, we will prove this result by
multiple induction.
In what follows, we assume that the statement
holds for 
all graphs with either fewer vertexes and no more hyperedges than $G$,
or fewer hyperedges and no more vertexes than $G$.
We also assume that for a particular $G$, the statement holds for all
hierarchy decompositions of $G$ that have fewer nodes than $T$.

There are five possibilities:
\begin{enumerate}
  \item The root $r$ of $T$ is labeled with $\chi(r) = \emptyset$, and $r$
    has no children.
  \item The root $r$ of $T$ is labeled with $\chi(r) = \emptyset$, and $r$
    has one child.
  \item The root $r$ of $T$ is labeled with $\chi(r) = \emptyset$, and $r$
    has two or more children, at least one of which, $x$, is also labeled with
    $\chi(x) = \emptyset$.
  \item The root $r$ of $T$ is labeled with $\chi(r) = \emptyset$, and $r$ has
    two or more children, none of which are labeled with the empty set.
  \item The root $r$ of $T$ is labeled with $\chi(r) \ne \emptyset$.
\end{enumerate}
We consider these cases separately.

\paragraph{Case 1}
If $r$ is labeled with the empty set and has no children, then it follows from
condition (1) that there are no hyperedges in $G$.  Therefore,
$\mathsf{hw}(G) = 0$.  Also, since $r$ is labeled with the empty set,
$\mathsf{w}(T) = 0$.  So, the statement holds in this case.

\paragraph{Case 2}
If $r$ has exactly one child, then by Lemma \ref{lemmaHDSingleChild}, there
is a hierarchy decomposition $\bar T$ of $G$ such that $\bar T$ has fewer nodes
than $T$ and $\mathsf{w}(T) = \mathsf{w}(\bar T)$.  By the inductive hypothesis,
\[
  \mathsf{w}(\bar T) \ge \mathsf{hw}(G).
\]
So, the statement holds in this case.

\paragraph{Case 3}
If $r$ is labeled with the empty set and has two or more children, at least one
of which is also labeled with the
empty set, then by Lemma \ref{lemmaHDSameLabel}, there is a hierarchy
decomposition $\bar T$ of $G$ such that $\bar T$ has fewer nodes
than $T$ and $\mathsf{w}(T) = \mathsf{w}(\bar T)$.  By the inductive hypothesis,
\[
  \mathsf{w}(\bar T) \ge \mathsf{hw}(G).
\]
So, the statement holds in this case.

\paragraph{Case 4}
Consider the case where $r$ is labeled with the empty set and has two or more
children, none of which are labeled with the empty set.  Let
$T_1, T_2, \ldots, T_l$ be the labeled subtrees rooted at the children of $r$.
Notice that if a hyperedge $e$ appears in the labeling of some node of $T_i$,
then it can't appear in the labeling of any node of $T_j$ for $i \ne j$, since
otherwise by condition (2) it must appear in the labeling
of the root node $r$, and $r$ has an empty labeling.  Therefore, the hyperedges
of $G$ are partitioned among the subtrees $T_i$.  Furthermore, condition (3)
implies that hyperedges associated with different subtrees are disconnected,
and therefore that graph $G$ is disconnected.  Therefore, if we let $G_i$
denote the $m$ connected components of $G$.
\[
  \mathsf{hw}(G) = \max_{i \le m} \mathsf{hw}(G_i).
\]
Now, let $\bar G_i$, for $i \le l$, denote the subgraph of $G$ that consists
of the nodes that are connected to a hyperedge in a labeling of a node of $T_i$.
Clearly, the $\bar G_i$ will be disconnected, so
\[
  \mathsf{hw}(G) = \max_{i \le l} \mathsf{hw}(\bar G_i).
\]
Each $\bar G_i$ will also have fewer vertexes and hyperedges than $G$, so
by the inductive hypothesis,
\[
  \mathsf{w}(T_i) \ge \mathsf{hw}(\bar G_i).
\]
Therefore,
\[
  \mathsf{w}(T)
  =
  \max_{i \le l} \mathsf{w}(T_i)
  \ge
  \max_{i \le l} \mathsf{hw}(\bar G_i)
  =
  \mathsf{hw}(G);
\]
so the statement holds in this case.

\paragraph{Case 5}
Finally, consider the case where the root node $r$ is labeled with some
hyperedge $e$.  Let $\bar G$ be the graph that results from removing $e$ from
$G$, and let $\bar T$ be the hierarchy decomposition that results from removing
$e$ from all the labellings of $T$.  Clearly, $\bar T$ will be a hierarchy
decomposition of $\bar G$.  Furthermore, $\bar G$ has fewer hyperedges and the
same number of vertexes as $G$, so by the inductive hypothesis,
\[
  \mathsf{w}(\bar T) \ge \mathsf{hw}(\bar G).
\]
Furthermore, since removing $e$ decreases the size of each of the labelings of
$\bar T$ by $1$,
\[
  \mathsf{w}(\bar T) = \mathsf{w}(T) - 1.
\]
Therefore,
\[
  \mathsf{w}(\bar T) \ge \mathsf{hw}(\bar G) + 1 \ge \mathsf{hw}(G),
\]
so the statement holds in this case.

Therefore, the statement holds in all cases, so the lemma follows from
induction on all hypergraphs and all hierarchy decompositions.
\end{proof}

\begin{lemma}
\label{lemmaHierarchyDecompBackward}
For any factor graph $G$, there exists a hierarchy decomposition $T$ of $G$
such that
\[
  \mathsf{w}(T) = \mathsf{hw}(G).
\]
\end{lemma}
\begin{proof}
As in the proof of Lemma \ref{lemmaBoundASG}, we will prove this result by
multiple induction.
In what follows, we assume that the statement
holds for 
all graphs with either fewer vertexes and no more hyperedges than $G$,
or fewer hyperedges and no more vertexes than $G$.

There are three possibilities:
\begin{enumerate}
  \item $G$ has no hyperedges.
  \item $G$ is disconnected.
  \item $G$ is connected.
\end{enumerate}
We consider these cases separately.

\paragraph{Case 1}
If $G$ has no hyperedges, then the tree $T$ with a single node labeled with the
empty set is a hierarchy decomposition for $G$.  It will satisfy
\[
  \mathsf{w}(T) = 0.
\]
Also, by (\ref{eqnHWBase}), for the case with no hyperedges,
\[
  \mathsf{hw}(G) = 0,
\]
so the statement holds in this case.

\paragraph{Case 2}
Assume that $G$ is disconnected, and its connected components are $G_i$.
Then by (\ref{eqnHWDisconnected}),
\[
  \mathsf{hw}(G) = \max_i \mathsf{hw}(G_i).
\]
Since each $G_i$ has fewer nodes and hyperedges than $G$, by the inductive
hypothesis there exists a hierarchy decomposition $T_i$ for each $G_i$ such
that
\[
  \mathsf{w}(T_i) = \mathsf{hw}(G_i).
\]
Let $T$ be the tree that has a root node labeled with the empty set, and
where the subtrees rooted at its children are exactly the $T_i$ above.
We now show that $T$ is a hierarchy decomposition for $G$ by validating the
conditions.
\begin{enumerate}
  \item Any hyperedge $e$ of $G$ must be a hyperedge of exactly one $G_i$. Since
    $T_i$ is a hierarchy decomposition of $G_i$, by condition (1) it follows
    that $e$ appears in the labeling of some node of $G_i$.  Therefore, $e$
    will appear in the labeling of the corresponding node of $G$, and the 
    condition holds.
  \item Again, any hyperedge $e$ of $G$ must be a hyperedge of exactly one
    $G_i$.  Therefore, if $e$ appears in two nodes of $T$, those two nodes
    must both be part of the same subtree $T_i$.  The condition follows from
    the corresponding condition of $T_i$.
  \item Since $G$ is disconnected, any pair of hyperedges in $G$ that share
    a vertex must both be part of exactly one $G_i$.  The condition then
    follows from the corresponding condition of $T_i$.
  \item This condition follows directly from the fact that for any $X$,
    $\emptyset \subseteq X$, and from the corresponding condition for the 
    subgraphs $T_i$.
\end{enumerate}
Therefore $T$ is a hierarchy decomposition for $G$.  Furthermore, since the
labelings of nodes of $T$ are the union of the labelings of nodes of $T_i$,
\[
  \mathsf{w}(T) = \max_{i} \mathsf{w}(T_i).
\]
Therefore,
\[
  \mathsf{hw}(G)
  =
  \max_i \mathsf{hw}(G_i)
  =
  \max_i \mathsf{w}(T_i)
  =
  \mathsf{w}(T),
\]
so the statement holds in this case.

\paragraph{Case 3}
Assume that $G$ is connected. Then by (\ref{eqnHWConnected}),
\[
  \mathsf{hw}(G) = 1 + \min_{e \in E} \mathsf{hw}(\langle N, E - \{e\}\rangle).
\]
Let $b$ be a hyperedge that minimizes this quantity, and let $\bar G$ be
the graph that results from removing this hyperedge from $G$.  Then,
\[
  \mathsf{hw}(G) = 1 + \mathsf{hw}(\bar G\rangle).
\]
Now, $\bar G$ has fewer hyperedges than $G$, so by the inductive hypothesis,
there exists a hierarchy decomposition $\bar T$ of $\bar G$ such that
\[
  \mathsf{hw}(\bar G) = \mathsf{w}(\bar T).
\]
Let $T$ be the tree that results from adding edge $b$ to every labeling of a
node of $\bar T$.  Clearly,
\[
  \mathsf{w}(T) = \mathsf{w}(\bar T) + 1.
\]
We now show that $T$ is a hierarchy decomposition for $G$ by validating the
conditions.
\begin{enumerate}
  \item Any hyperedge $e$ of $G$ is either hyperedge $b$ or some hyperedge in
    $\bar G$.  If it is $b$, then it must appear in a labeling of a node of $T$
    since it appears in all such labelings.  Otherwise, the condition follows
    from the corresponding condition of $\bar T$.
  \item Again, any hyperedge $e$ of $G$ is either hyperedge $b$ or some
    hyperedge in $\bar G$.  If it is $b$, then the condition follows from the
    fact that all node labelings of $T$ contain $b$.  Otherwise, the condition
    follows from the corresponding condition of $\bar T$.
  \item For any pair of edges $(e, f)$ of $G$, either $b \notin \{e, f\}$, or
    $b \in \{e, f\}$.  If $b \in \{e, f\}$, then the condition follows from
    condition (1) of $\bar T$ and the fact that all node labelings of $T$
    contain $b$; otherwise, the condition follows from the corresponding
    condition of $\bar T$.
  \item This condition follows directly from the fact that for any $X$ and $Y$,
    if $X \subseteq Y$, then $X \cup \{ b \} \subseteq Y \cup \{ b \}$.
\end{enumerate}
Therefore $T$ is a hierarchy decomposition for $G$, and
\[
  \mathsf{w}(T)
  =
  \mathsf{w}(\bar T) + 1
  =
  \mathsf{hw}(\bar G) + 1
  =
  \mathsf{hw}(G),
\]
so the statement holds in this case.

Therefore the statement holds in all cases, so the lemma follows from induction
over all factor graphs.
\end{proof}

Next, we restate and prove Statement \ref{stmtHierarchyDecomp}.

\stmtHierarchyDecomp*
\begin{proof}
This statement follows directly from Lemmas \ref{lemmaHierarchyDecompForward}
and \ref{lemmaHierarchyDecompBackward}.
\end{proof}

Now, we provide a definition for hypertree width, and prove the comparison
statements from the body of the paper.

\begin{definition}[Hypertree Decomposition~\cite{gottlob2014treewidth}]
A \emph{hypertree decomposition} of a hypergraph $G$ is a structure
$\langle T, \rho, \chi \rangle$, where $\langle  T, \rho \rangle$ is a tree
decomposition for the primal graph of $G$, $\chi$ is a labeling of the nodes of
$T$ with hyperedges of $\chi$, and the following conditions are satisfied:
(1) for any node $u$ of $T$, and for all vertices $x \in \rho(u)$, there exists
an edge $e \in \chi(u)$ such that $x \in e$; and
(2) for any node $u$ of $T$, for any edge $e \in \chi(u)$, and for any
descendant $v$ of $u$ in $T$,
$e \cap \rho(v) \subseteq \rho(u)$.

The \emph{width} of a hypertree decomposition, denoted $\mathsf{w}(T)$, is
defined (as in the hierarchy decomposition case) to be the cardinality of the
largest hyperedge-labeling of a node of $T$.  That is,
\[
  \mathsf{w}(T) = \max_u \Abs{\chi(u)}.
\]
The \emph{hypertree width} $\mathsf{tw}(G)$ of a graph $G$ is the minimum width
among all hypertree decompositions of $G$. 
\end{definition}

We restate and prove Statement \ref{stmtHierarchyWidthComparison}.

\stmtHierarchyWidthComparison*
\begin{proof}
Consider the hierarchy decomposition $T$ of $G$ that satisfies
$\mathsf{w}(T) = \mathsf{hw}(G)$.  Assume that we augment $T$ with an additional
node labeling $\rho(v)$ that consists of the vertexes of $G$ that are part of
at least one edge of $\chi(v)$ (that is, $\rho(v) = \cup \chi(v)$).  We show
that $T$ is a hypertree decomposition of $G$.

First, we must show that $\langle T, \rho \rangle$ is a tree decomposition for
the primal graph of $G$; we do so by verifying the conditions independently.
\begin{enumerate}
  \item Each vertex of $G$ must appear in the labeling of some node of $T$,
    since (by assumption) it must appear in some hyperedge of $G$, and each
    hyperedge appears in the labeling of some node of $T$ because $T$ is a
    hierarchy decomposition of $G$.
  \item Similarly, any two vertexes $(x, y)$ that are connected by an edge in
    the primal graph of $G$ must appear together in some hyperedge of $G$.
    Since each hyperedge of $G$ appears in the edge-labeling of some node of
    $T$, it follows that $x$ and $y$ must both appear in the vertex-labeling of
    that node.
  \item Finally, assume that vertex $x$ appears in the labeling of two
    different nodes $u$ and $v$ of $T$.  Since $x \in \rho(u)$, there must
    exist an edge $e \in \chi(u)$ such that $x \in e$; similarly, there must
    be an edge $f \in \chi(v)$ such that $x \in f$.  By condition (3) of
    the definition of hierarchy decomposition, there must exist a node $w$ of
    $T$ such that $\{ e, f \} \subseteq \chi(w)$.  By condition (2) of the
    definition of hierarchy decomposition, for all nodes $z$ on the path
    between $u$ and $w$, $e \in \chi(z)$; similarly, for all nodes $z$ on the
    path between $v$ and $w$, $f \in \chi(z)$.  Now, any node $z$ on the path
    between $u$ and $v$ must be either part of the path between $u$ and $w$ and
    the path between $v$ and $w$, therefore either $e \in \chi(z)$ or
    $f \in \chi(z)$.  It follows that $x \in \rho(z)$, as desired.
\end{enumerate}
Therefore $T$ is a tree decomposition for the primal graph of $G$.

Next, we show that the $T$ is also a hypertree decomposition for $G$ by
verifying the conditions independently.
\begin{enumerate}
  \item For any node $u$ of $T$, and for all vertices $x \in \rho(u)$, there
    must exist an edge $e \in \chi(u)$ such that $x \in e$, because
    $\rho(u) = \cup \chi(v)$.  Therefore the condition holds.
  \item For any node $u$ of $T$ and for any
    descendant $v$ of $u$ in $T$, we know from condition (4) of the definition
    of hierarchy decomposition that $\chi(u) \subseteq \chi(v)$.  So,
    for any edge $e \in \chi(u)$, it also holds that $e \in \chi(v)$.
    Therefore,
    $e \cap \rho(v) = e \subseteq \rho(u),$ and so the condition holds.
\end{enumerate}
We conclude that $T$ is a hypertree decomposition for $G$.  But, its width
is $\mathsf{hw}(G)$.  Since the hypertree width of $G$ is the minimum width
among all hypertree decompositions of $G$, it follows that
\[
  \mathsf{tw}(G) \le \mathsf{hw}(G),
\]
which is the desired expression.
\end{proof}

Next, we restate and prove Statement \ref{stmtHierarchyWidthKPoly}.

\stmtHierarchyWidthKPoly*
\begin{proof}
Consider
Algorithm \ref{algComputeHWK}, which computes this quantity for a fixed $k$.

\begin{algorithm}[h]
  \caption{$\mathtt{HW}(G, k)$: Compute if $\mathsf{hw}(G) \le k$}
  \begin{algorithmic}
  \label{algComputeHWK}
    \IF{$G$ has no edges}
      \RETURN $\mathsf{true}$
    \ENDIF
    \IF{$k = 0$}
      \RETURN $\mathsf{false}$
    \ENDIF
    \FOR{$e \in \mathsf{edges}(G)$}
      \STATE Let $\bar G$ be the graph that results from removing $e$ from the
        connected component of $G$ that contains $e$
      \IF{$\mathtt{HW}(\bar G, k - 1)$}
        \RETURN $\mathsf{true}$
      \ENDIF
    \ENDFOR
    \RETURN $\mathsf{false}$
  \end{algorithmic}
\end{algorithm}

Each execution of $\mathtt{HW}$ requires at most linear time in the number of
hyperedges of $G$ to compute its connected components.  If we let $e$ be the
number of hyperedges of $G$, it also requires at most
$e$ executions of $G$ with parameter $k - 1$.  Therefore, $\mathtt{HW}$ will
run in time $O(e^k)$, which is polynomial in the number of hyperedges, as
desired.
\end{proof}

As an aside here, we note that the hierarchy width can be expressed in terms
of an existing graph parameter, the
\emph{tree-depth}~\citesec[p. 115]{de2012sparsity}, which we denote
$\mathsf{td}(G)$. To do this, we let
$L(G)$, the \emph{line graph} of $G$, denote the graph such that every factor
of $G$ is a node of $L(G)$, and two nodes of $L(G)$ are connected by an edge
if and only if their corresponding factors share a dependent variable.
Using this, it is trivial to show (by definition) that
\[
  \mathsf{hw}(G) = \mathsf{td}(L(G)).
\]
Since it is a known result that it is possible to compute whether a graph
has tree-depth at most $k$ in time polynomial (in fact, linear) in the number
of nodes of the graph, we could have also used this to prove Statement
\ref{stmtHierarchyWidthKPoly}; we avoided doing this to make the result more
accessible.

\begin{restatable}{statement}{stmtHierarchyWidthDegree}
\label{stmtHierarchyWidthDegree}
The hierarchy width of a factor graph $G$ is greater than or equal to the
maximum degree of a variable in $G$.
\end{restatable}

Finally, we restate and prove Statement \ref{stmtHierarchyWidthDegree}.

\stmtHierarchyWidthDegree*
\begin{proof}
The statement follows by induction.  Considering the parts of the definition
of hierarchy width individually,
removing a single hyperedge from a hypergraph only decreases the maximum
degree of the hypergraph by at most $1$, and splitting the hypergraph into
connected components doesn't change the maximum degree.
\end{proof}

\section{Proofs of Factor Graph Template Results}
In this section, we prove the results in Section
\ref{ssHierarchicalFactorGraphs}.  First, we prove
Lemma \ref{lemmaHierarchicalHW}.

\begin{restatable}{lemma}{lemmaHierarchicalHW}
\label{lemmaHierarchicalHW}
If $G$ is an instance of a hierarchical factor graph template $\mathcal{G}$
with $E$ template factors, then $\mathsf{hw}(G) \le E$.
\end{restatable}

\lemmaHierarchicalHW*
\begin{proof}
We prove this result using the notion of
hierarchy decomposition from the previous section.
For the instantiated factor graph $G$, let $T$ be a tree where each node
$v$ of $T$ is associated with a particular assignment of the first $n$
head symbols of the rules. That is, there is a node of $T$ for each
tuple of object symbols $(x_1, \ldots, x_m)$---even for $m = 0$.
(Since the rules are hierarchical, and therefore
must contain the same symbols in the same order, these assignments are
well-defined.)

Label each node $v$ of $T$ with the set containing all
the instantiated factors of $G$ that agree with this assignment of
head variables.
That is, if the head variable assignment for a factor $\phi$ is
$(y_1, \ldots, y_p)$, then it will be in the labeling of the node
$v = (x_1, \ldots, x_m)$ if and only if $m \ge p$ and for all $i \le p$,
$y_i = x_i$.  It is not hard to show that this is a valid hierarchy
decomposition for this factor graph.

The width of this hierarchy decomposition is at most the number of
template factors $E$ because each node $v$ can only possibly contain a single
instantiated factor for each template factor.  (This is because for each
node, only one possible assignment of head variables is compatible with
that node.)  The Lemma now follows from an application of Statement
\ref{stmtHierarchyDecomp}.
\end{proof}

Next, we restate and prove Statement \ref{stmtHierarchicalFG}.

\stmtHierarchicalFG*
\begin{proof}
If we use either logical or ratio semantics, the maximum factor weight will be
bounded with $M = O(\log n)$.  We furthermore know from Lemma
\ref{lemmaHierarchicalHW} that for a particular
hierarchical template, the hierarchy width is bounded independent of the
dataset.  So, $h M = O(\log n)$, and the Statement now follows directly from
an application of Theorem \ref{thmPolynomialMixingTime} 
\end{proof}

% \subsubsection*{References}

\bibliographystylesec{plainnat} 
\bibliographysec{references}

}{}

\end{document}